\documentclass[11pt]{article}
\usepackage[letterpaper,top=0.78in,bottom=0.82in,left=0.88in,right=0.88in]{geometry}
\usepackage[T1]{fontenc}
\usepackage[utf8]{inputenc}
\usepackage{newtxtext}
\usepackage{amsmath,amssymb,amsthm,mathtools,needspace}
\usepackage{newtxmath}
\usepackage[scaled=0.92]{helvet}
\usepackage[varqu]{inconsolata}
\usepackage{booktabs,tabularx,array,longtable}
\usepackage{graphicx}
\usepackage{tikz}
\usetikzlibrary{arrows.meta,positioning,fit,calc,backgrounds,shapes.geometric,decorations.pathreplacing,matrix}
\usepackage{xcolor}
\usepackage{microtype}
\usepackage{enumitem}
\usepackage{caption}
\usepackage{float}
\usepackage[numbers,sort&compress]{natbib}
\usepackage{titlesec}
\usepackage{xurl}
\usepackage{placeins}
\usepackage{setspace}
\usepackage[hidelinks]{hyperref}

\definecolor{FFBlue}{HTML}{4C78A8}
\definecolor{FFBlueLight}{HTML}{EAF1F8}
\definecolor{FFGold}{HTML}{D3A52F}
\definecolor{FFGoldLight}{HTML}{FBF4DF}
\definecolor{FFRed}{HTML}{B56470}
\definecolor{FFRedLight}{HTML}{F8ECEE}
\definecolor{FFGray}{HTML}{6A7179}
\definecolor{FFGrayDark}{HTML}{3F454B}
\definecolor{FFRule}{HTML}{D3D7DB}
\definecolor{FFLight}{HTML}{F7F8F9}
\definecolor{FFInk}{HTML}{111315}
\color{FFInk}
\setlist{leftmargin=*,topsep=3pt,itemsep=1.8pt,parsep=0pt,partopsep=0pt}
\newcolumntype{L}[1]{>{\raggedright\arraybackslash}p{#1}}
\newcolumntype{Y}{>{\raggedright\arraybackslash}X}
\hypersetup{
  pdftitle={Intrinsic-Extrinsic Coupling in Learning Dynamics},
  pdfauthor={Qinyou Wang},
  pdfsubject={Observation-relative hidden learning structure, intrinsic-extrinsic coupling, and coordination},
  pdfkeywords={intrinsic-extrinsic coupling, hidden learning state, observation-relative fiber, continuation-conditioned value, coordination},
  pdfcreator={LaTeX with TikZ/PGF},
  pdfdisplaydoctitle=true
}
\titleformat{\section}{\large\bfseries}{\thesection}{0.68em}{}
\titleformat{\subsection}{\normalsize\bfseries}{\thesubsection}{0.60em}{}
\titleformat{\subsubsection}{\normalsize\bfseries\itshape}{\thesubsubsection}{0.52em}{}
\titlespacing*{\section}{0pt}{12pt plus 3pt minus 2pt}{4.5pt plus 1pt minus 1pt}
\titlespacing*{\subsection}{0pt}{8.5pt plus 2pt minus 1pt}{2.8pt plus 1pt minus 1pt}
\titlespacing*{\subsubsection}{0pt}{6.5pt plus 1.5pt minus 1pt}{2.2pt}

\newtheoremstyle{ffplain}{6pt}{6pt}{\itshape}{}{\bfseries}{.}{0.5em}{}
\newtheoremstyle{ffdefinition}{6pt}{6pt}{\normalfont}{}{\bfseries}{.}{0.5em}{}
\newtheoremstyle{ffremark}{5.5pt}{5.5pt}{\normalfont}{}{\itshape}{.}{0.5em}{}
\theoremstyle{ffdefinition}
\newtheorem{definition}{Definition}[section]
\theoremstyle{ffplain}
\newtheorem{proposition}{Proposition}[section]
\newtheorem{corollary}[proposition]{Corollary}
\theoremstyle{ffremark}

\makeatletter
\renewcommand{\maketitle}{%
  \begin{center}
    {\fontsize{19}{22.5}\selectfont\bfseries \@title\par}
    \vspace{1.15em}
    {\normalsize \@author\par}
  \end{center}
  \vspace{0.50em}
}
\makeatother

\renewenvironment{abstract}{%
  \begin{center}\begin{minipage}{0.92\linewidth}\small
  \begin{center}\bfseries Abstract\end{center}\vspace{-0.45em}
}{%
  \end{minipage}\end{center}\vspace{0.35em}
}

\numberwithin{equation}{section}
\newcommand{\Id}{\mathrm{Id}}
\newcommand{\R}{\mathbb{R}}

\newcommand{\norm}[1]{\left\lVert#1\right\rVert}
\tikzset{
  ffbox/.style={
    draw=FFGrayDark,
    line width=0.66pt,
    rounded corners=1.1pt,
    align=center,
    inner xsep=2.7mm,
    inner ysep=1.8mm,
    font=\sffamily\footnotesize,
    text=FFInk
  },
  ffneutral/.style={ffbox,fill=FFLight},
  ffblue/.style={ffbox,fill=FFBlueLight},
  ffgold/.style={ffbox,fill=FFGoldLight},
  ffred/.style={ffbox,fill=FFRedLight},
  ffwhite/.style={ffbox,fill=white},
  ffgray/.style={ffbox,fill=FFLight,draw=FFRule},
  ffdashframe/.style={
    draw=FFGray,
    line width=0.58pt,
    rounded corners=1.2pt,
    dashed,
    fill=white,
    inner sep=2.4mm,
    align=center,
    font=\sffamily\footnotesize,
    text=FFInk
  },
  ffarrow/.style={
    -{Latex[length=1.85mm,width=1.12mm]},
    draw=FFGrayDark,
    line width=0.70pt,
    rounded corners=1.0pt
  },
  ffarrowblue/.style={
    -{Latex[length=1.85mm,width=1.12mm]},
    draw=FFBlue,
    line width=0.78pt,
    rounded corners=1.0pt
  },
  ffarrowgold/.style={
    -{Latex[length=1.85mm,width=1.12mm]},
    draw=FFGold,
    line width=0.78pt,
    rounded corners=1.0pt
  },
  ffarrowred/.style={
    -{Latex[length=1.85mm,width=1.12mm]},
    draw=FFRed,
    line width=0.78pt,
    rounded corners=1.0pt
  },
  ffdasharrow/.style={
    -{Latex[length=1.72mm,width=1.04mm]},
    draw=FFGray,
    line width=0.58pt,
    dashed,
    rounded corners=1.0pt
  },
  ffline/.style={draw=FFGrayDark,line width=0.64pt},
  ffrule/.style={draw=FFRule,line width=0.58pt},
  ffdashline/.style={draw=FFGray,line width=0.56pt,dashed},
  ffpanel/.style={font=\sffamily\scriptsize\bfseries,text=FFGrayDark,anchor=west},
  fflabel/.style={font=\sffamily\scriptsize,text=FFInk,align=center},
  ffsubtle/.style={font=\sffamily\tiny,text=FFGray,align=center},
  ffmath/.style={font=\footnotesize,text=FFInk,align=center},
  ffdot/.style={circle,minimum size=2.5pt,inner sep=0pt,draw=none,fill=FFGrayDark}
}

\title{Intrinsic--Extrinsic Coupling in Learning Dynamics}
\author{Qinyou Wang}
\date{}
\begin{document}
\maketitle
\thispagestyle{plain}

\begin{abstract}
A learner's current observations need not determine its response to further training. We formulate intrinsic--extrinsic coupling through the continuation-conditioned value of a constrained learning-state intervention, with observation-relative fibers describing present agreement. An executable finite-frame classifier-head write protects current logits while repairing specified historical margins under finite-precision acceptance checks. We distinguish local admissibility, continuation-conditioned intervention value, and complete-policy performance. A matched four-cell contrast identifies readout-specific non-additivity between the same intrinsic intervention and alternative external continuations. In a CLINC-derived class-incremental setting, replay changes the write's 32-update contribution from five correct predictions to zero. Nonzero interactions also occur under output distillation, with a RoBERTa backbone, and under optimizer-native momentum stochastic gradient descent with decoupled weight decay (SGDW). Under SGDW, correct-count interactions are negative in all three activated roots at 128 updates, showing that coupling need not imply positive synergy. The mathematical analysis distinguishes feasible local repairs and favorable terminal outputs from training-reachable repair regions. Separate coordination tests show that content controls match or exceed the development gain, while a five-root fresh-test comparison with Fiber present in every arm shows root-dependent rather than uniformly beneficial correct-count effects. On the secondary cross-entropy readout, guided allocation yields lower mean loss than standard replay in all five pairs. Together, these results make intrinsic--extrinsic coupling operational by connecting executable state geometry to continuation-conditioned value, matched interaction identification, and closed-loop coordination, while separating identified coupling from complete-policy performance.
\end{abstract}

\section{Introduction}
A learner's present behavior need not determine how it will learn next. \emph{Fiber Fingerprints} formalizes this distinction through controlled future-learning responses within present-behavior equivalence classes \cite{p1}. We use \emph{intrinsic learning dynamics} for the evolution and state-dependent training response of the learner's internal structure. ``Hidden'' is relative to a chosen observation, not a separate autonomous process. Observation-relative fibers are one description of these distinctions; intrinsic dynamics are the research object, not a synonym for that description.

We study direct intervention on learning state followed by continued external training. An intrinsic intervention changes selected state coordinates subject to a declared observation constraint. An extrinsic continuation specifies the subsequent training rule. Both operate on the same learner: an internal intervention changes the starting state for external updates, and those updates change the states available for later intervention. The two interfaces therefore act on the same evolving learner.

We formulate the question through \emph{continuation-conditioned value} and a readout-specific interaction. Fix a parent, one internal proposal, a future horizon, external randomness, and an endpoint utility. Continuing the identity/execute pair under each of two training policies gives four outcomes. Their mixed contrast asks whether the continuation changes the intervention's marginal value---equivalently, whether the intervention changes that continuation's incremental value. These are two readings of one forward-time experiment. Figure~\ref{fig:coupling} distinguishes this controlled identification from the feedback decisions of a complete coordinator.

A finite-frame Fiber actuator makes the internal interface executable. Its ideal head write protects a fixed current feature frame while repairing selected historical margins; the rounded implementation must pass explicit checks. Replay supplies the first extrinsic factor \cite{er}. At one common parent, the same write adds five correct predictions without replay and none with replay after 32 updates, although a margin difference persists. Further four-cell comparisons extend the observation to distillation based on Learning without Forgetting (LwF) \cite{lwf}, RoBERTa \cite{roberta}, and momentum stochastic gradient descent (SGD) with decoupled weight decay (SGDW). In the optimizer-native study, the construction and task phase are fixed, while each optimizer generates its own parent trajectory and legal activation. Nonzero interaction does not require identical trajectories, signs, or effect sizes.

We connect admissible state changes, their value under specified future learning, and the performance of policies that repeatedly coordinate the two. Observation-relative geometry determines what a finite intervention preserves; a matched continuation determines what it changes in future outcomes; a complete-policy test determines whether a fixed coordination rule is useful across training histories. We separately test a fixed replay-allocation rule under a common periodic Fiber actuator. On the development root, a replay-reinforcement rule reaches 953 rather than 950 correct predictions with the same Fiber consultation schedule in both policies and equal replay counts, but registered random-content and permuted-signal controls match or exceed it. Across five fresh roots on an unused 1,500-item test split, the correct-count effect of the fixed guided allocation rule varies by root rather than remaining uniformly beneficial. Lower mean cross-entropy (CE) under guided allocation in all five roots is a secondary result, not a replacement for the primary outcome. Thus the matched branches identify coupling, and the five-root policy test characterizes the root-dependent value of a fixed allocation rule.

The mathematical analysis distinguishes locally feasible parameter repairs and favorable terminal outputs from training-reachable regions of favorable future control. Neither a local feasibility result nor a favorable output path identifies a long-horizon repair basin or shows that Fiber-selected times are uniquely valuable. Similar terminal accuracy does not require internal-state convergence. The local repair set has an exact affine-ball characterization, and saved logits certify favorable output paths for the development controls.

State/output distinctions, sequential value, and factorial contrasts are classical \cite{hk,bellman,rubin,factorial}; null-space editing and edit retention under subsequent training are established topics \cite{alphaedit,retention}. \emph{Revelation Control} already studies priced intervention choice, productive reuse, and state-dependent continuation value \cite{p2}. The experiments establish scoped instances, not a complete evolution law or benchmark-wide performance advantage. Our contribution is to connect these elements operationally through a specified observation constraint, an implemented state intervention, matched future-learning comparisons, and explicit policy-attribution boundaries.

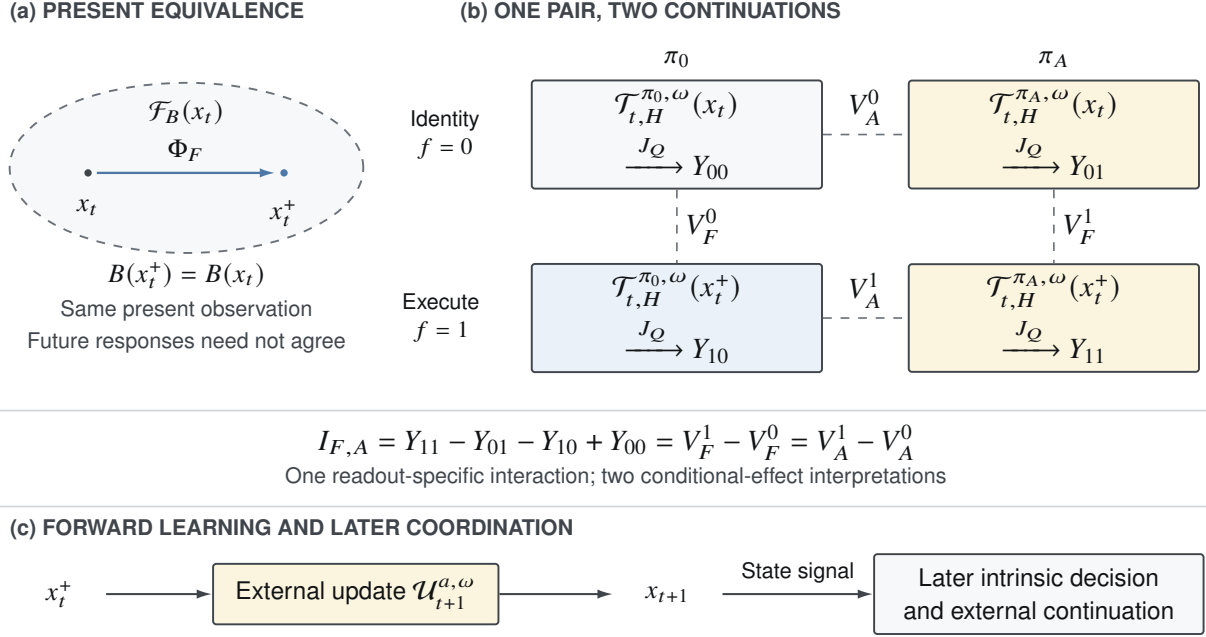
\begin{figure}[!t]
\centering
\resizebox{0.96\linewidth}{!}{
\begin{tikzpicture}[x=1cm,y=1cm]
  \path[use as bounding box] (0,0) rectangle (14.80,7.70);
  \node[ffpanel] at (0,7.46) {(a) PRESENT EQUIVALENCE};
  \node[ffpanel] at (5.40,7.46) {(b) ONE PAIR, TWO CONTINUATIONS};
  \draw[draw=FFGray,line width=.65pt,dashed,fill=FFLight]
    (2.25,5.60) ellipse [x radius=2.12cm,y radius=1.02cm];
  \node[ffmath] at (2.25,6.27) {$\mathcal F_B(x_t)$};
  \coordinate (x0) at (1.07,5.54);
  \coordinate (x1) at (3.42,5.54);
  \node[ffdot] at (x0) {};
  \node[ffdot,fill=FFBlue] at (x1) {};
  \draw[ffarrowblue,shorten <=3pt,shorten >=3pt] (x0)--node[above,ffmath]{$\Phi_F$}(x1);
  \node[ffmath,anchor=north] at (1.07,5.35) {$x_t$};
  \node[ffmath,anchor=north] at (3.42,5.35) {$x_t^+$};
  \node[ffmath] at (2.25,4.31) {$B(x_t^+)=B(x_t)$};
  \node[fflabel,text=FFGrayDark] at (2.25,3.89) {Same present observation};
  \node[fflabel,text=FFGrayDark] at (2.25,3.51) {Future responses need not agree};
  \node[ffmath] at (8.13,6.90) {$\pi_0$};
  \node[ffmath] at (12.65,6.90) {$\pi_A$};
  \node[fflabel,anchor=east] at (5.90,6.00) {Identity\\$f=0$};
  \node[fflabel,anchor=east] at (5.90,3.80) {Execute\\$f=1$};
  \node[ffneutral,minimum width=3.48cm,minimum height=1.02cm,inner ysep=1.1mm,font=\small] (y00) at (8.13,6.00)
    {$\mathcal T_{t,H}^{\pi_0,\omega}(x_t)$\\[2pt]$\xrightarrow{\ J_Q\ }Y_{00}$};
  \node[ffgold,minimum width=3.48cm,minimum height=1.02cm,inner ysep=1.1mm,font=\small] (y01) at (12.65,6.00)
    {$\mathcal T_{t,H}^{\pi_A,\omega}(x_t)$\\[2pt]$\xrightarrow{\ J_Q\ }Y_{01}$};
  \node[ffblue,minimum width=3.48cm,minimum height=1.02cm,inner ysep=1.1mm,font=\small] (y10) at (8.13,3.80)
    {$\mathcal T_{t,H}^{\pi_0,\omega}(x_t^+)$\\[2pt]$\xrightarrow{\ J_Q\ }Y_{10}$};
  \node[ffgold,minimum width=3.48cm,minimum height=1.02cm,inner ysep=1.1mm,font=\small] (y11) at (12.65,3.80)
    {$\mathcal T_{t,H}^{\pi_A,\omega}(x_t^+)$\\[2pt]$\xrightarrow{\ J_Q\ }Y_{11}$};
  \draw[ffdashline] (y00.east)--node[above=2pt,ffmath,font=\small,inner sep=1pt]{$V_A^0$}(y01.west);
  \draw[ffdashline] (y10.east)--node[above=2pt,ffmath,font=\small,inner sep=1pt]{$V_A^1$}(y11.west);
  \draw[ffdashline] (y00.south)--node[right=2pt,ffmath,font=\small,inner sep=1pt]{$V_F^0$}(y10.north);
  \draw[ffdashline] (y01.south)--node[right=2pt,ffmath,font=\small,inner sep=1pt]{$V_F^1$}(y11.north);
  \draw[ffrule] (0,2.69)--(14.80,2.69);
  \node[ffmath,font=\small] at (7.40,2.31)
    {$I_{F,A}=Y_{11}-Y_{01}-Y_{10}+Y_{00}=V_F^1-V_F^0=V_A^1-V_A^0$};
  \node[fflabel,text=FFGrayDark] at (7.40,1.89) {One readout-specific interaction; two conditional-effect interpretations};
  \draw[ffrule] (0,1.54)--(14.80,1.54);
  \node[ffpanel] at (0,1.27) {(c) FORWARD LEARNING AND LATER COORDINATION};
  \node[ffmath,minimum width=1.1cm] (post) at (.73,.49) {$x_t^+$};
  \node[ffgold,minimum width=3.25cm,minimum height=.67cm] (ordinary) at (4.28,.49)
    {External update $\mathcal U_{t+1}^{a,\omega}$};
  \node[ffmath,minimum width=1.3cm] (next) at (8.02,.49) {$x_{t+1}$};
  \node[ffneutral,minimum width=3.95cm,minimum height=.72cm] (later) at (12.47,.49)
    {Later intrinsic decision\\and external continuation};
  \draw[ffarrow] (post.east)--(ordinary.west);
  \draw[ffarrow] (ordinary.east)--(next.west);
  \draw[ffarrow] (next.east)--node[above,fflabel]{State signal}(later.west);
\end{tikzpicture}}
\caption{\textbf{From observation-relative structure to interaction and coordination.} (a) An ideal intrinsic move lies in a declared present-observation fiber; the outline asserts neither manifold structure nor completeness of that description. (b) The identity/execute pair is continued under two external rules at fixed parent, proposal, horizon, randomization, and readout. Dashed connectors compare outcomes, not branch states evolving into one another. The two conditional-effect readings belong to one contrast. (c) A complete coordinator can use later state signals to change future decisions. All evolution is forward in time; subsequent writes are disabled only in the isolated experiment (b).}
\label{fig:coupling}
\end{figure}

\section{Intrinsic--extrinsic coupling: an observation-relative framework}
\label{sec:framework}
\subsection{Full state and observation-relative fibers}
Let $x_t$ denote the full learning state after ordinary update $t$ and its associated state operations, before an intrinsic decision. It includes the model and optimizer state, counters, data cursor, random-generator and precision state, and any memory, teacher, or coordinator state needed to determine subsequent evolution. In the implementation below, reservoir admission precedes the intrinsic decision. Diagnostic logs are not learning-state coordinates. Neither the chosen observation nor a scalar structural score is assumed sufficient for future learning. Random-generator state records endogenous execution history; $\omega$ below supplies the declared external inputs and keyed randomization convention.

The full state has both discrete and continuous components. Differential statements below are made on a fixed-dimensional continuous chart, with task/head configuration and other discrete components held fixed. Head births and reservoir admissions are handled by the transition map, not by differentiating across a change of state dimension.

\Needspace{11\baselineskip}
\begin{definition}[Observation-relative fiber]
For a specified observation map $B$, the exact fiber through $x$ is
\begin{equation}
 \mathcal F_B(x)=\{x':B(x')=B(x)\}.
 \label{eq:fiber}
\end{equation}
On a continuous chart where $B$ is differentiable, its infinitesimally invisible subspace is
\begin{equation}
 \boxed{\mathcal K_x=\ker DB_x=\{v:DB_xv=0\}.}
 \label{eq:tangent}
\end{equation}
\end{definition}
If $B$ is $C^1$ and has locally constant rank near $x$, the local fiber is a submanifold with tangent space $\mathcal K_x$. Without such regularity, Equation~\eqref{eq:tangent} remains a linearized constraint, not an identification of a smooth fiber. Even in the regular case, a straight finite write satisfying $DB_xv=0$ need not satisfy $B(x+v)=B(x)$. Exact finite preservation requires an additional construction. Section~\ref{sec:actuator} supplies one by restricting the action to a head-parameter slice on which the selected observation is linear.

``Hidden'' is therefore observation-relative. The Fiber Fingerprint description distinguishes equality under $B$ from equality under controlled future responses \cite{p1}. For fixed $B$, the level set is well defined, but the choice of $B$ and this descriptive approach are not asserted to exhaust the internal structure. A nonzero element of $\mathcal K_x$ is not thereby guaranteed to be visible under every continuation or useful for every readout. Neither future visibility nor positive value is included in the definition of the subspace. A present-observation fiber is neither an all-future indistinguishability class \cite{hk} nor necessarily invariant under training: a common continuation may separate its members. The response fingerprint in \cite{p1} contains future-response information beyond the level set alone.

\subsection{Intrinsic actuation and extrinsic continuation}
An intrinsic intervention acts on learning state; it is not the dynamics themselves. A proposal $F$ specifies a finite write and its observation and protected-state constraints. Let $\mathscr L_t(x,F)$ be the acceptance predicate. An accepted proposal induces $\Phi_F$, while $\Phi_0=\Id$. The certificate concerns present-state constraints; Section~\ref{sec:actuator} realizes them on a fixed feature frame.

For an intrinsic decision $f_t$, followed by an external learning action $a_{t+1}$, the execution order is
\begin{equation}
 x_t^+=\Phi_t^{f_t}(x_t),\qquad
 x_{t+1}=\mathcal U_{t+1}^{a_{t+1},\omega_{t+1}}(x_t^+).
 \label{eq:ordered}
\end{equation}
For the binary interface, $\Phi_t^{0}=\Id$ and $\Phi_t^{1}=\Phi_{F_t}$ for the proposal made at step $t$. The map $\mathcal U$ includes the next ordinary update and its associated state operations. A coordinator may use a current internal-state signal to choose $a_{t+1}$ or a later action; it cannot change the already executed update $t$.

``Intrinsic'' and ``extrinsic'' identify two operational interfaces, not independent physical subsystems. Both act on the same learner and may use overlapping historical information; their effects need not be orthogonal. This terminology is unrelated to intrinsic rewards and is not a taxonomy of all learning algorithms. The Fiber transaction realizes a constrained state reset between ordinary updates, not a controllability theorem.

Two forms of feedback are possible in Equation~\eqref{eq:ordered}. External learning changes the states on which later intrinsic decisions are made. Conversely, an intrinsic write changes the state from which an external rule operates; a coordinator can also explicitly use the accompanying signal to choose that rule. These are architectural possibilities. Whether their effects are nonzero is an empirical question.

\subsection{Continuation-conditioned intervention value}
Fix a step $t$, a pre-action parent $x$, and one accepted proposal $F$ constructed at that parent. Write $\pi$ for a fully specified future external policy. The continuation $\mathcal T_{t,H}^{\pi,\omega}$ executes the next $H$ ordinary updates, $t+1,\ldots,t+H$, with subsequent intrinsic interventions disabled; $\mathcal T_{t,0}^{\pi,\omega}=\Id$. The policy can be state dependent: the same rule need not produce identical branch-local updates or auxiliary states.

Let $J_Q$ be a scalar endpoint utility, with larger values preferred. The specification $Q$ includes the evaluation items and prediction/readout convention; in particular, the active class set must be declared. Correct count and negative CE are different utilities. A structural margin can be used as a separate readout, but is not interchangeable with correct count.

\begin{definition}[Pathwise single-action value]
At fixed $(t,x,F,H,\pi,\omega,Q)$, define
\begin{equation}
 \begin{aligned}
 V_F(x,H;\pi,\omega,Q)
 &=J_Q\!\left(\mathcal T_{t,H}^{\pi,\omega}(\Phi_F(x))\right)\\
 &\quad-J_Q\!\left(\mathcal T_{t,H}^{\pi,\omega}(x)\right).
 \end{aligned}
 \label{eq:value}
\end{equation}
\end{definition}
Policy-conditioned value is standard in sequential decision theory \cite{bellman,sutton}; Equation~\eqref{eq:value} specializes it to one constrained write. The branches share the exact parent and original proposal; identity does not solve a different repair. Matched $\omega$ fixes external inputs and randomization conventions, not post-treatment gradients, optimizer moments, or adaptive outputs. Equal dropout seeds need not yield identical batched masks when the batch changes.

An expected value additionally requires a distribution over $\omega$ and, when relevant, parents. Each mechanism contrast is pathwise. Independent-root summaries are reported only for the explicitly declared study designs in Section~\ref{sec:protocol}. In the replay realization, ``on'' means the declared retained schedule, and ``off'' removes replay but leaves current-task learning active.

When horizons use different seen-task panels $Q_H$, the reported object is $V_F(x,H;\pi,\omega,Q_H)$. Comparing its signs at two horizons is legitimate for that declared endpoint family, but does not isolate elapsed learning time from a changing evaluation population. A pure fixed-readout temporal claim would require a common $Q$ and common prediction convention.

A nonzero repair target, acceptance by $\mathscr L_t$, and positive $V_F$ are three different predicates. The first identifies a frame-level constraint violation, the second checks an admissible write, and the third compares future outcomes. Geometry and the acceptance test settle the first two questions; an explicit continuation and readout are needed for the third.

\subsection{Two readings of a four-cell interaction}
Fix $(t,x,F,H,\omega,Q)$ and two fully specified external policies $\pi_0,\pi_A$. Index the intrinsic decision by $f\in\{0,1\}$ and the external policy by $a\in\{0,1\}$, with $\Phi_1=\Phi_F$ and $\pi_1=\pi_A$. Here $a$ selects a continuation policy, whereas $a_{t+1}$ in Equation~\eqref{eq:ordered} is one update-level action. Define the four outcomes
\begin{equation}
 Y_{fa}=J_Q\!\left(\mathcal T_{t,H}^{\pi_a,\omega}(\Phi_f(x))\right).
 \label{eq:fouroutcomes}
\end{equation}
These are potential outcomes under two controlled intervention factors \cite{rubin,factorial}. The intrinsic-intervention effect under policy $a$ and the external-policy effect after decision $f$ are, respectively,
\begin{equation}
 V_F^a=Y_{1a}-Y_{0a},\qquad V_A^f=Y_{f1}-Y_{f0}.
 \label{eq:twoeffects}
\end{equation}
\begin{definition}[Readout-specific intrinsic--extrinsic interaction]
\begin{equation}
 \boxed{I_{F,A}=Y_{11}-Y_{01}-Y_{10}+Y_{00}.}
 \label{eq:interaction}
\end{equation}
\end{definition}
\begin{proposition}[Standard factorial contrast identity]
\label{prop:bilateral}
For any four real-valued outcomes,
\begin{equation}
 \boxed{I_{F,A}=V_F^1-V_F^0=V_A^1-V_A^0,}
 \label{eq:bilateral}
\end{equation}
and, for $f,a\in\{0,1\}$,
\begin{equation}
 Y_{fa}=Y_{00}+f\,(Y_{10}-Y_{00})+a\,(Y_{01}-Y_{00})+fa\,I_{F,A}.
 \label{eq:decomp}
\end{equation}
Thus $I_{F,A}=0$ exactly when this four-cell outcome table is additive in the two intervention indicators.
\end{proposition}
\begin{proof}
Expanding either difference of effects in Equation~\eqref{eq:bilateral} gives Equation~\eqref{eq:interaction}. Equation~\eqref{eq:decomp} agrees with the four outcomes at $(0,0),(1,0),(0,1),(1,1)$. An additive table must have a zero mixed contrast, and a zero mixed contrast removes its $fa$ term.
\end{proof}

This elementary identity is useful because it gives two equally valid readings of one measured interaction: external continuation modifies an internal action's value, and the internal action modifies the incremental value of that external continuation. These are not two independent replications, and no causal influence backwards in time is implied. The physical order remains Equation~\eqref{eq:ordered}.

With matched four-cell interventions, a nonzero $I_{F,A}$ is evidence of non-additive effects on $J_Q$ at this parent. If larger $J_Q$ is better, positive $I$ is more-than-additive utility, not a guarantee that either individual effect or the joint effect is positive. Zero $I$ means additivity of this table only, not independence of states or equality of logits, other readouts, or other horizons. A nonlinear change of utility scale can change the contrast. In particular, nonzero correct-count interaction is not itself a proof of a smooth mixed derivative, a noncommuting update operator, or a universal dynamical coupling law; an explicit algebraic example is given in Appendix~\ref{app:contrast}.

\subsection{Single-action values versus closed-loop values}
A complete coordinator $\Pi$ specifies both the intrinsic decisions and the subsequent external allocation rules. Its state includes whatever finite history the policy uses. Starting from a common root and coupled $\omega$, define the terminal policy contrast
\begin{equation}
 \Delta J_Q(\Pi,\Pi_0;\omega)
 =J_Q(X_T^{\Pi,\omega})-J_Q(X_T^{\Pi_0,\omega}).
 \label{eq:policy}
\end{equation}
Here later parents, frames, accepted proposals, teacher snapshots, and allocations may differ. Holding the actuation algorithm and consultation schedule fixed does not hold the realized writes fixed after trajectories diverge. This is not Equation~\eqref{eq:value} with only a larger $H$. In particular, a local reinforcement fork and a full trajectory with subsequent reinforcement windows are different interventions.

One-shot effects measured at different parents cannot generally be summed to obtain a policy effect. For an exact telescoping construction, specify nested whole policies $\Pi^{(0)},\ldots,\Pi^{(m)}$ under a common root, $\omega$, and $Q$:
\begin{equation}
 \Delta J_Q(\Pi^{(m)},\Pi^{(0)};\omega)
 =\sum_{i=1}^{m}\Delta J_Q(\Pi^{(i)},\Pi^{(i-1)};\omega).
 \label{eq:telescoping}
\end{equation}
The summands depend on that nesting and its continuation rules. The identity does not imply context-free values for the individual actions. Reverting to the baseline rule after a previous intervention also does not restore the baseline state; identity is a present no-write decision, not a future quality certificate.

\subsection{Coordination under an explicit workload constraint}
For replay, the resource vector actually matched in the full coordination comparison is
\begin{equation}
 W_R(\Pi)=\bigl(N_{\mathrm{replay\ events}},N_{\mathrm{replay\ examples}}\bigr).
 \label{eq:workload}
\end{equation}
In the replay realization, a coordinator can redirect examples using a pre-action state signal while preserving $W_R$. The objective of higher terminal utility at fixed $W_R$ is distinct from lower workload at fixed utility. ``Closed-loop'' means that subsequent allocation depends on a signal acquired from the learner; it does not assert a stability theorem or an identified reduced-state model. Neither equates total cost: Fiber sensing, numerical certification, data selection, hardware execution, and elapsed time remain separate quantities. Matching teacher-forward counts in the second-rule application is another external-workload match, not a proof of total-compute equality.

Writing $\max_\Pi J_Q$ subject to $W_R(\Pi)\le W_0$ formulates a control objective; it does not assert that the tested coordinator solves it. Likewise, an improved complete policy is not by itself an identification of which part of the policy is necessary. Separating a state write from signal-only allocation, or testing the specificity of one Fiber-aligned allocation, requires corresponding controls. Section~\ref{sec:coord} reports attribution controls frozen after the initial development outcome and before their own outcomes, followed by independent-root coordination tests. These controls assess complete-policy coordination and selected components of Fiber-signal use separately from the matched one-shot interaction.

\section{Finite-frame realization of intrinsic actuation}
\label{sec:actuator}
\subsection{The exact head-space zero constraint}
We realize the intrinsic interface with a finite-frame classifier-head repair. Stack the active classifier weights and biases in $\Theta\in\R^{q\times d}$, with each encoder feature augmented by a last coordinate equal to one. In this implementation $d=769$ and $q$ is the number of currently active classes. The encoder is fixed during the transaction.

At entry, acquire one immutable inference-mode frame. Its current rows form $C\in\R^{n_c\times d}$ and its historical representative rows form $\mathsf H\in\R^{n_h\times d}$. The typeface distinguishes the historical matrix $\mathsf H$ from the future horizon $H$. Throughout the action, $C$, $\mathsf H$, the labels, and the active class configuration are held fixed. They may be reacquired at the next consultation.

The protected observation on this head slice is
\begin{equation}
 B_C(\Theta)=C\Theta^\top.
 \label{eq:BC}
\end{equation}
Therefore the finite-write action subspace is
\begin{equation}
 \boxed{\mathcal K_C=\{\Delta\Theta\in\R^{q\times d}:C\Delta\Theta^\top=0\}.}
 \label{eq:czero}
\end{equation}
Here $C$ is a feature matrix and $\Delta\Theta$ a classifier-head parameter increment. Because Equation~\eqref{eq:BC} is linear in $\Theta$,
\begin{equation}
 B_C(\Theta+\Delta\Theta)-B_C(\Theta)=C\Delta\Theta^\top
 \label{eq:exactfinite}
\end{equation}
holds exactly at any amplitude in real arithmetic. The fiber on this slice is the affine space $\Theta+\mathcal K_C$, whose dimension is $q(d-\operatorname{rank}C)$.

Figure~\ref{fig:finiteframe} specifies the finite mechanism within the general framework of Figure~\ref{fig:coupling}; $\Delta\Theta_*$ denotes the ideal repair in Proposition~\ref{prop:leastnorm}. This is the precise bridge from Equation~\eqref{eq:tangent} to a finite action: for fixed $C$, the derivative kernel and finite preserving-write subspace coincide. It does not hold by freezing only the symbol $B$ while recomputing a different frame after the action. Nor does it claim preservation on all network inputs; historical representative outputs are deliberately allowed to change.

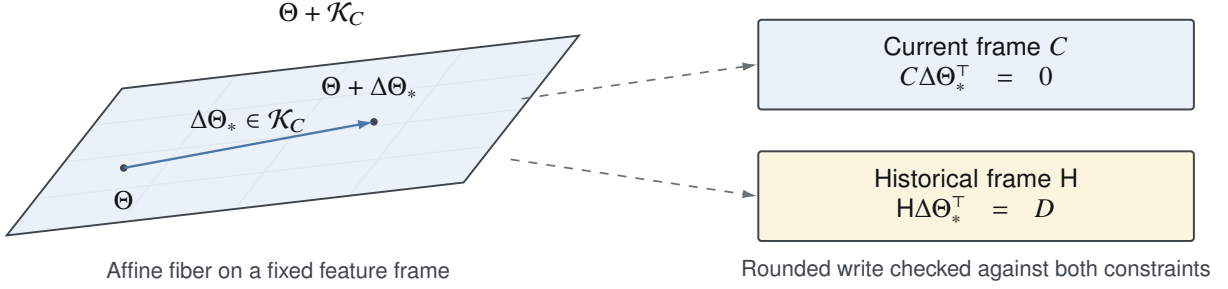
\begin{figure}[!htbp]
\centering
\resizebox{0.96\linewidth}{!}{
\begin{tikzpicture}[x=1cm,y=1cm]
  \path[use as bounding box] (0,0) rectangle (14.45,4.10);
  \node[ffpanel] at (0,3.87) {FINITE HEAD WRITE: ONE STATE CHANGE, TWO READOUT CONSTRAINTS};
  \coordinate (A) at (.35,.72); \coordinate (B) at (5.70,1.34);
  \coordinate (C) at (7.05,3.06); \coordinate (D) at (1.70,2.44);
  \fill[FFBlueLight] (A)--(B)--(C)--(D)--cycle;
  \foreach \t in {.25,.5,.75}{
    \draw[draw=FFBlue!18,line width=.35pt] ($(A)!\t!(B)$)--($(D)!\t!(C)$);
    \draw[draw=FFBlue!18,line width=.35pt] ($(A)!\t!(D)$)--($(B)!\t!(C)$);
  }
  \draw[ffline] (A)--(B)--(C)--(D)--cycle;
  \node[ffmath,anchor=south] at (4.05,3.04) {$\Theta+\mathcal K_C$};
  \coordinate (th0) at (1.72,1.51); \coordinate (th1) at (4.65,2.05);
  \node[ffdot] at (th0) {}; \node[ffdot] at (th1) {};
  \node[ffmath,anchor=north] at (1.72,1.38) {$\Theta$};
  \node[ffmath,anchor=south] at (4.60,2.18) {$\Theta+\Delta\Theta_*$};
  \draw[ffarrowblue] (th0)--node[above,ffmath]{$\Delta\Theta_*\in\mathcal K_C$}(th1);
  \node[fflabel,text=FFGrayDark] at (3.53,.31) {Affine fiber on a fixed feature frame};
  \node[ffblue,minimum width=5.10cm,minimum height=1.05cm,text width=4.55cm] (cur) at (11.70,2.72)
    {Current frame $C$\\[-1pt]$C\Delta\Theta_*^{\top}=0$};
  \node[ffgold,minimum width=5.10cm,minimum height=1.05cm,text width=4.55cm] (hist) at (11.70,1.18)
    {Historical frame $\mathsf H$\\[-1pt]$\mathsf H\Delta\Theta_*^{\top}=D$};
  \draw[ffdasharrow] (6.38,2.32)--(cur.west);
  \draw[ffdasharrow] (6.25,1.61)--(hist.west);
  \node[fflabel,text=FFGrayDark] at (11.70,.31) {Rounded write checked against both constraints};
\end{tikzpicture}}
\caption{\textbf{A finite-frame realization of intrinsic actuation.} The ideal head write moves within the affine fiber $\Theta+\mathcal K_C$. The same move gives zero current-frame change and the prescribed historical increment $D$; dashed arrows report these readouts, not additional parameter motions. The plane represents this linear head slice only. The rounded implementation must pass the numerical certificate in Section~\ref{sec:actuator}, rather than inherit exact preservation from the schematic.}
\label{fig:finiteframe}
\end{figure}

\subsection{A prescribed centered margin target}
For $q\ge2$ active classes and historical representative $i$ with label $y_i$, define its minimum margin against all other active classes:
\begin{equation}
 m_i=\min_{j\ne y_i}\bigl[(\mathsf H\Theta^\top)_{i,y_i}-(\mathsf H\Theta^\top)_{i,j}\bigr].
 \label{eq:margin}
\end{equation}
The specified target uses $\mu=1$ and $\eta=2^{-10}$:
\begin{equation}
 \begin{aligned}
 d_i&=\begin{cases}0,&m_i\ge\mu,\\ \mu+\eta-m_i,&m_i<\mu,\end{cases}\\
 D_{i,:}&=d_i\left(e_{y_i}-q^{-1}\mathbf 1\right)^\top,
 \qquad D\in\R^{n_h\times q}.
 \end{aligned}
 \label{eq:target}
\end{equation}
The target is zero-sum across classes, $D\mathbf 1=0$. For a deficient representative, the correct logit increment is $d_i(1-1/q)$ and every competing increment is $-d_i/q$; each correct-minus-competitor difference therefore increases by exactly $d_i$. This is a specified interpolation target, not a validation-loss objective or a claim that a class deserves more replay.

Representatives come from the actual post-admission reservoir: for each represented old class, the item with the smallest unique identifier (UID) is selected, with fixed duplicate and tie-breaking rules. Missing old classes are recorded, not replaced by new external samples. The predicate $D=0$ only describes these selected representatives. A targeted repair is not a class-confined parameter perturbation: the centered target and its minimum-norm realization can affect many head coordinates and unconstrained inputs. The accompanying margin/deficiency signal is a present-frame diagnostic, not a measurement of the full future-response fingerprint.

\subsection{Null-space interpolation and the least-norm write}
Assume $C$ has full row rank and define
\begin{equation}
 P_C=I-C^\top(CC^\top)^{-1}C,\qquad N=\mathsf H P_C.
 \label{eq:projector}
\end{equation}
$P_C$ is the orthogonal projector onto $\ker C$. The residual rows of $N$ are the parts of the historical frame available for modification without changing the current frame. Full row rank of $N$ is a further assumption; full row rank of $C$ alone does not guarantee it. In particular, it requires $n_h\le d-n_c$ and, in this case, is equivalent to full row rank of the stacked matrix $[C^\top\;\mathsf H^\top]^\top$.

Under these assumptions the prescribed ideal map is
\begin{equation}
 \boxed{\Delta\Theta_* =D^\top(NN^\top)^{-1}N.}
 \label{eq:repair}
\end{equation}
This is a full-row-rank instance of minimum-norm linear interpolation \cite{penrose}. The implementation uses the corresponding linear solves, not explicit inverses. It does not substitute a pseudoinverse, ridge term, or an alternative solver when the prescribed numerical path fails.

\begin{proposition}[Exact finite-frame realization]
\label{prop:leastnorm}
If $C$ and $N$ have full row rank, Equation~\eqref{eq:repair} is the unique minimum-Frobenius-norm solution of
\begin{equation}
 C\Delta\Theta^\top=0,\qquad \mathsf H\Delta\Theta^\top=D.
 \label{eq:constraints}
\end{equation}
For the target in Equation~\eqref{eq:target}, all deficient represented margins become $\mu+\eta$, and the logits of already nondeficient representatives are unchanged. The solution also satisfies $\Delta\Theta_*^\top\mathbf1=0$.
\end{proposition}
The proof appears in Appendix~\ref{app:proof}. The last equality says that the ideal write does not shift the mean active-class logit at any fixed feature row. This centering is not an all-input preservation guarantee. The minimum-norm statement concerns the prescribed $D$ and the two exact equality constraints; it is not global optimality among all margin-improving, generalization-improving, or budget-feasible interventions.

Rank deficiency does not imply that every possible target is infeasible. With the projector defined above, an exact target is feasible precisely when every column of $D$ belongs to $\operatorname{range}N$. Full row rank is the sufficient condition that makes every target feasible. The implementation requires a certified nonsingular solve; the more general feasibility condition does not alter this acceptance rule.

\subsection{A feasible repair set, not a unique feasible point}
\begin{corollary}[Finite-frame repair set]
\label{cor:repairset}
Under Proposition~\ref{prop:leastnorm}, put $A=[C^\top\;\mathsf H^\top]^\top$ and $r_A=\operatorname{rank}A$. If $r\ge\norm{\Delta\Theta_*}_F$, all writes satisfying Equation~\eqref{eq:constraints} and the ideal budget $\norm{\Delta\Theta}_F\le r$ form
\begin{equation}
 \boxed{\mathcal S_{\rm frame}(r)=
 \left\{\Delta\Theta_*+Z:\ AZ^\top=0,\quad
 \norm Z_F\le\sqrt{r^2-\norm{\Delta\Theta_*}_F^2}\right\}.}
 \label{eq:repairset}
\end{equation}
The underlying affine solution space has dimension $q(d-r_A)$. With the additional centering constraint $Z^\top\mathbf1=0$, its dimension is $(q-1)(d-r_A)$. The bounded set is empty if $r<\norm{\Delta\Theta_*}_F$, and a singleton if equality holds; the stated dimensions apply to its relative interior when the radius is positive.
\end{corollary}
\begin{proof}
The difference of two feasible writes lies in the homogeneous kernel of $A$. The least-norm solution is orthogonal to that kernel, so
$\norm{\Delta\Theta_*+Z}_F^2=\norm{\Delta\Theta_*}_F^2+\norm Z_F^2$.
Each of the $q$ head rows has $d-r_A$ free coordinates. Centering imposes one sum constraint for each such coordinate. Substitution gives the result.
\end{proof}
This is a direct linear-algebraic consequence of the least-norm construction, not a new principle of null-space editing. At the stored first common attribution frame, step 351, $q=30$, $d=769$, and $r_A=28$, giving 22,230 free coordinates, or 21,489 under centering. The recomputed ideal norm is approximately $0.06859953<1$; exact rank certificates for the stored dyadic matrices support the dimension calculation (Appendix~\ref{app:outputgeometry}). The implementation chooses the prescribed minimum-norm write, not arbitrary members of this set. Other members need not pass native acceptance checks or improve future outcomes. In particular, a large frame-feasible set is not evidence for a training-reachable favorable basin.

\Needspace{14\baselineskip}
\subsection{The write budget depends on target and geometry}
\begin{proposition}[Geometry-conditioned repair cost]
\label{prop:cost}
Under the assumptions of Proposition~\ref{prop:leastnorm}, with $n_h>0$,
\begin{equation}
 \norm{\Delta\Theta_*}_F^2
 =\operatorname{tr}\!\left(D^\top(NN^\top)^{-1}D\right)
 =\sum_{i=1}^{n_h}\frac{\norm{\ell_i^\top D}_2^2}{\sigma_i(N)^2},
 \label{eq:geocost}
\end{equation}
where $\ell_i$ are orthonormal left singular vectors of $N$. Hence
\begin{equation}
 \frac{\norm{D}_F}{\sigma_{\max}(N)}
 \le \norm{\Delta\Theta_*}_F
 \le \frac{\norm{D}_F}{\sigma_{\min}(N)}.
 \label{eq:costbounds}
\end{equation}
\end{proposition}
This identity, proved in Appendix~\ref{app:costproof}, separates the requested margin correction from its geometric accessibility. Two frames with the same target magnitude can require different write norms. A large norm can reflect poor separation of historical features from the protected current span, not just large deficiencies.

In exact arithmetic, if $\norm{\Delta\Theta_*}_F>r$, no write of norm at most $r$ satisfies this particular exact interpolation target. The implemented rejection is more conservative: the numerical routine demands certified upper bounds for both ideal-map and proposal norms within the unit cap. Failure of this test does not by itself give a certified lower bound above the cap, prove infeasibility of every approximate repair, or establish the absence of any useful action. A write-cap rejection should be described as failure to certify the prescribed local map within the specified budget.

\subsection{From the ideal map to a certified finite transaction}
The implemented write is a finite-precision transaction. Stored inference features and head values are FP32; the proposal computation uses their exact promotion to CPU FP64. Interval enclosures, solve checks, and exact-dyadic checks on the stored frame are used before native staged/live head evaluations. These levels should not be conflated: exact real arithmetic on stored FP32 values is not an exact statement about every GPU matrix multiplication or every network input.

For the stored-frame part of an accepted transaction, let $\widehat{\Delta\Theta}$ denote the actual head difference after finite-precision rounding. The certificate includes
\begin{equation}
 \begin{aligned}
 \norm{C\widehat{\Delta\Theta}^{\top}}_{\max}&\le\varepsilon_C,\\
 \norm{\mathsf H\widehat{\Delta\Theta}^{\top}-D}_{\max}&\le\varepsilon_H,\\
 \norm{\widehat{\Delta\Theta}}_F&\le1,
 \end{aligned}
 \label{eq:rounded}
\end{equation}
with $\varepsilon_C=\varepsilon_H=2^{-12}$ and $\norm{\cdot}_{\max}$ the largest absolute matrix entry. It separately checks unchanged current-frame argmax and old represented margins at least $\mu$. Native staged and live logits are then checked for the same current prediction preservation, current logit tolerance, old-margin threshold, and per-item/mean current CE tolerance $2^{-11}$. Current argmax preservation is not inferred from a small residual alone.

\begin{corollary}[What the interpolation tolerance controls]
\label{cor:roundedmargin}
For the exact affine evaluations of stored values, Equation~\eqref{eq:rounded} implies
\begin{equation}
 \bigl|\widehat m_i-(m_i+d_i)\bigr|\le2\varepsilon_H.
 \label{eq:marginerror}
\end{equation}
A deficient representative therefore satisfies $\widehat m_i\ge1+2^{-11}>1$ under the specified $\eta$ and $\varepsilon_H$. For a previously nondeficient representative the residual bound alone only gives $\widehat m_i\ge m_i-2\varepsilon_H$; the explicit post-margin check remains necessary at the threshold.
\end{corollary}
A companion real-arithmetic bound is $|\Delta\operatorname{CE}_i|\le2\varepsilon_C$ for the same stored logits and labels. It explains the relation between the two tolerance scales but does not replace the implementation's native CE check. Both statements are derived in Appendix~\ref{app:tolerance}.

An accepted transaction changes active head weights and biases only. The encoder parameters, optimizer moments and counters, gradient buffers, reservoir, RNG state, and module modes are protected during the transaction. A recoverable rejection restores the original learning state before returning. A fatal device failure or unverified restoration is not reported as a successful identity return. Numerical checks authenticate the current transaction; they do not certify future accuracy or the usefulness of a coordinator.

\Needspace{8\baselineskip}
\subsection{Same-frame saturation is not future invariance}
\begin{corollary}[Saturation on an unchanged frame]
\label{cor:saturation}
After the ideal repair in Proposition~\ref{prop:leastnorm}, recomputing Equation~\eqref{eq:target} on the identical historical frame and updated head yields $D=0$. The same conclusion holds for an accepted rounded write whose stored-frame represented margins have all passed the explicit $m_i\ge\mu$ postcheck.
\end{corollary}
This corollary concerns one immutable frame with no intervening ordinary update, head growth, or representative change. It does not imply identity at the next Fiber consultation. Missing classes and unconstrained inputs are outside the statement.

There is also no general implication from protected current logits to a protected future learning update. For a differentiable same-frame loss $\ell(C\Theta^\top)$, let $G$ be its derivative with respect to logits. An exact null-space write leaves logits and $G$ unchanged, but the feature derivative changes by
\begin{equation}
 \nabla_C\ell\big|_{\Theta+\Delta\Theta}-\nabla_C\ell\big|_{\Theta}
 =G\Delta\Theta,
 \label{eq:featuregradient}
\end{equation}
which need not vanish. If those features depend on encoder parameters, backpropagation can consequently differ. This is a fixed-frame chain-rule observation, not a claim that the measured training-mode forward is identical; the actual action frame is acquired in inference mode. Appendix~\ref{app:gradient} gives the precise distinction. Thus a finite observation-preserving action can be relevant to subsequent learning without a promised direction or sign of future utility.

\subsection{Scheduled Fiber actuation}
An abstract scheduled actuator is specified by an eligibility set $\mathcal E$ and the target/acceptance logic above. We use a consultation interval of $K=32$ updates, with
\begin{equation}
 \mathcal E_{32}=\{191+32k:k=0,\ldots,19\}.
 \label{eq:eligibility}
\end{equation}
The first eligible point is 191: old tasks exist from step 160, and the current-frame construction uses the eight-example epoch tail. The interval is fixed across all comparisons, rather than optimized separately for each external training rule.

At an eligible point the mechanism reacquires the frame, constructs the target, and returns one of three principal outcomes: saturated identity when $D=0$; a nonzero commit when the prescribed repair passes the complete certificate; or rejected identity when it cannot be certified. Empty historical support is separately reported as identity. Outside $\mathcal E_{32}$ there is no consultation. In compact form, on the learning-state coordinates,
\begin{equation}
 x_t^+=
 \begin{cases}
  \Phi_{F_t}(x_t),&t\in\mathcal E_{32},\ D_t\ne0,\ \mathscr L_t(x_t,F_t)=1,\\
  x_t,&\text{otherwise, for a normally returned decision.}
 \end{cases}
 \label{eq:scheduled}
\end{equation}
A regular offer uses at most one extra encoder call over the combined frame and no extra backward pass. Committing and abstaining can both incur sensing/certification work. The code and numerical parameters are retained across replay and distillation. ``Sparse'' describes the opportunity schedule, not a guarantee that only a few eligible points commit in every environment. The observed commit times, recurring demand, and rejection patterns are outcomes, not scripted triggers or universal regime labels.

\FloatBarrier

\FloatBarrier
\section{Experimental design and identification scope}
\label{sec:protocol}
\subsection{Common learning task and readouts}
All studies use the same deterministically selected 50-class subset of CLINC \cite{clinc}, partitioned into five tasks of ten classes. This is not a full 150-class benchmark. The subset has 5,000 training and 1,000 validation examples. Each task is trained for five epochs with current batch size 32 and an eight-item epoch tail, giving 160 updates per task and 800 in total. Task indices are zero based. BERT-base-cased with a growing linear head is the default \cite{bert}; the backbone extension uses RoBERTa-base and its tokenizer \cite{roberta}. Appendix~\ref{app:protocol} gives implementation settings; exact source revisions and data bindings were documented in internal verification records.

Evaluation concatenates all active heads and predicts over all seen classes without a task oracle. Correct count $J$ is the primary utility; $-\mathrm{CE}$ is a distinct secondary utility. Validation task ends are steps 159, 319, 479, 639, and 799, with 200, 400, 600, 800, and 1,000 items. The independent five-root policy study additionally uses all 1,500 test examples from the same 50 classes, 30 per class, once at termination. This test split is separate from reused validation data and is not consulted for policy execution, calibration, stopping, or root selection. The same terminal test items are used across all five roots and three arms. A count from a growing panel is not compared across horizons as though its population were fixed.

The reservoir has 150 slots and is updated after ordinary training. Replay is drawn before admission and contributes 16 examples whenever the reservoir is nonempty: 799 events and 12,784 historical presentations in the full schedule. Replay and current examples are concatenated under a single mean-CE update. Replay-off removes the replay component, not current learning; it changes both batch composition and normalization. Matching keyed randomization does not imply identical dropout tensors after batch size changes. Thus the estimand concerns the implemented continuation, not a separately isolated additive replay gradient.

\subsection{Experimental units and complementary comparisons}
An isolated four-cell comparison restores one common pre-action parent and proposal, then disables every subsequent Fiber write. Within a cell quartet, endpoint items and prediction conventions agree. A whole-policy comparison instead starts at a common root and lets future states, writes, and allocations diverge. Within a backbone, a root changes keyed head initialization, data order, dropout, memory admission, and replay sampling, but not the pretrained encoder or class partition. The training root, rather than an evaluation item or a repeated horizon, is the unit of independent-root variation.

Table~\ref{tab:scope} organizes the experiments by the question each identifies. Continuation and backbone extensions retain the finite-frame construction while changing the learning context. The two principal optimizer comparisons ask different questions: whether interaction occurs under optimizer-native trajectories, and whether it also occurs under a controlled update-scale calibration. The allocation comparisons then test the added value of a fixed replay coordinator and the value of exact Fiber-label alignment relative to mask-matched random content; the five-root policy comparison retains Fiber in every arm.

\begin{table}[!htbp]\centering\small
\caption{\textbf{Scientific questions and experimental units.} Every nonactivation remains in its prescribed denominator. Multiple anchors and horizons do not create additional independent roots. All comparisons reuse the class partition; only the five-root policy study adds the previously unused test split.}
\label{tab:scope}
\begin{tabularx}{\linewidth}{L{0.21\linewidth}L{0.29\linewidth}Y}\toprule
Comparison & Parent/root design & Identified question and readout scope\\\midrule
Replay interaction & One development root; matched full- or 75\%-replay parents & Continuation-conditioned effects; validation with endpoint-dependent panels\\
Distillation interaction & Development root; first three legal commits on the distillation prefix & Distillation-off/on four cells; $H=32,128$\\
Backbone extension & Three RoBERTa roots; first legal commit by 671 & Replay interaction under a second backbone; $H=32,128$\\
Native-optimizer comparison & Three new roots, paired AdamW/SGDW; two task phases & Phase activation and conditional replay interaction; $H=32,128$\\
Update-scale sensitivity & Three distinct roots; fixed anchors and independent rate calibration & Interaction under sampled update-scale matching; associated diagnostics reuse these roots\\
Allocation attribution & Development root; three control trajectories and archived comparators & Complete-policy attribution on the reused 1,000-item validation endpoint\\
Fresh-root policy comparison & Five new roots, each with three full policies containing Fiber & Additional allocation value and content specificity on 1,500 test items; CE and validation secondary\\\bottomrule
\end{tabularx}
\end{table}

\subsection{Prospective decisions, nonactivation, and interpretation}
The initial development policies were exploratory. Follow-up controls and extensions were frozen before their own reported outcomes, after development evidence where applicable. ``Registered'' refers to these archived protocols, not to an independently timestamped public registry. Calibration roots are distinct from measurement roots, and calibration does not use Fiber-interaction outcomes. Root numbers are local to a study. The update-scale sensitivity study and its associated optimizer diagnostics reuse one root panel; the native-optimizer study uses a separate panel. Equal root numbers across those studies do not imply matched realizations. Within the native-optimizer study, the AdamW and SGDW conditions are explicitly paired by root construction and exogenous inputs.

Anchor selection uses only the prescribed pre-action rule. If no legal nonzero proposal occurs in the specified domain, the outcome is \emph{nonactivation}, not an activated zero interaction. No replacement root or later favorable anchor is selected by endpoint value. The native-optimizer study preserves an unactuated full-replay scout: each offer is evaluated on a reloaded clone, so measuring the earlier phase cannot alter the later parent. Activation and conditional interaction are reported separately.

For the five-root policy study, each primary contrast is called directionally positive only if at least four of five roots are positive and both its mean and median are positive; the negative rule is symmetric. Otherwise it is mixed or inconclusive. These are the frozen descriptive criteria, not population-level significance tests. CE remains secondary even when its signs are more consistent. Detailed per-root results, nonactivations, and workload are reported in the corresponding appendices.

\section{Controlled identification across learning contexts}
\label{sec:causal}
\subsection{One action, two external continuations}
The certificate in Section~\ref{sec:actuator} constrains a present write; Equations~\eqref{eq:value} and~\eqref{eq:interaction} concern its future effect. Replay supplies the first controlled realization of the extrinsic factor $A$, evaluated in a same-parent four-cell experiment. A single pre-action parent is taken from the full-replay trajectory at step 351. Its prescribed Fiber action repairs the label-11 representative. Execute and identity branches each undergo updates 352--383 with future replay either on or off. All later Fiber writes are disabled, and the four endpoints use the same 600-item, 30-class evaluation convention.

Figure~\ref{fig:fourcell} shows the four outcomes. Without future replay, the action contributes $500-495=5$ correct predictions; with full future replay it contributes $578-578=0$. Hence
\begin{equation}
 V_F^0=5,\qquad V_F^1=0,\qquad I_{F,R}=-5.
 \label{eq:h32interaction}
\end{equation}
Equivalently, replay contributes $83$ correct predictions after identity and $78$ after execute. These are the two readings of the same contrast in Proposition~\ref{prop:bilateral}, not independent replications. The same certified write therefore has a continuation-dependent correct-count value at the matched parent. This readout-specific non-additivity does not assign replay a generally harmful effect or either channel a context-free value.

\begin{figure}[!htbp]
\centering
\resizebox{0.96\linewidth}{!}{\begin{tikzpicture}[x=1cm,y=1cm]
  \path[use as bounding box] (0,0) rectangle (14.45,6.35);
  \node[ffpanel] at (0,6.10) {MATCHED FOUR-CELL EXPERIMENT: 32 UPDATES, 600 ITEMS};
  \node[fflabel,font=\sffamily\footnotesize\bfseries] at (5.00,5.39) {Replay off};
  \node[fflabel,font=\sffamily\footnotesize\bfseries] at (10.65,5.39) {Replay on};
  \node[fflabel,anchor=east] at (2.28,4.40) {Identity};
  \node[fflabel,anchor=east] at (2.28,2.43) {Execute};
  \node[ffneutral,minimum width=3.48cm,minimum height=1.04cm] (k0) at (5,4.40)
    {$Y_{00}$\quad {\large\bfseries 495}};
  \node[ffgold,minimum width=3.48cm,minimum height=1.04cm] (kr) at (10.65,4.40)
    {$Y_{01}$\quad {\large\bfseries 578}};
  \node[ffblue,minimum width=3.48cm,minimum height=1.04cm] (e0) at (5,2.43)
    {$Y_{10}$\quad {\large\bfseries 500}};
  \node[ffgold,minimum width=3.48cm,minimum height=1.04cm] (er) at (10.65,2.43)
    {$Y_{11}$\quad {\large\bfseries 578}};
  \draw[ffdashline,draw=FFBlue] (k0.south)--node[right,ffmath]{$V_F^0=\mathbf{+5}$}(e0.north);
  \draw[ffdashline] (kr.south)--node[right,ffmath]{$V_F^1=\mathbf{0}$}(er.north);
  \draw[ffdashline] (k0.east)--node[above,fflabel]{$+83$}(kr.west);
  \draw[ffdashline] (e0.east)--node[above,fflabel]{$+78$}(er.west);
  \draw[ffrule] (0,1.38)--(14.45,1.38);
  \node[ffmath] at (7.25,.89) {$I_{F,R}=V_F^1-V_F^0=0-5=\mathbf{-5}$};
  \node[fflabel,text=FFGrayDark] at (7.25,.27) {Replay removes the prediction advantage; a margin difference remains};
\end{tikzpicture}}
\caption{\textbf{External continuation changes the value of an identical Fiber action.} All four branches share the step-351 parent, proposal specification, 32-update horizon, and 600-item endpoint panel. Each cell is one matched branch outcome, not a seed average. Connectors denote scalar comparisons, not transitions between branch states. Vertical differences are Fiber values, with $I_{F,R}=-5$; the two $578$ counts do not imply equal logits or margins.}
\label{fig:fourcell}
\end{figure}
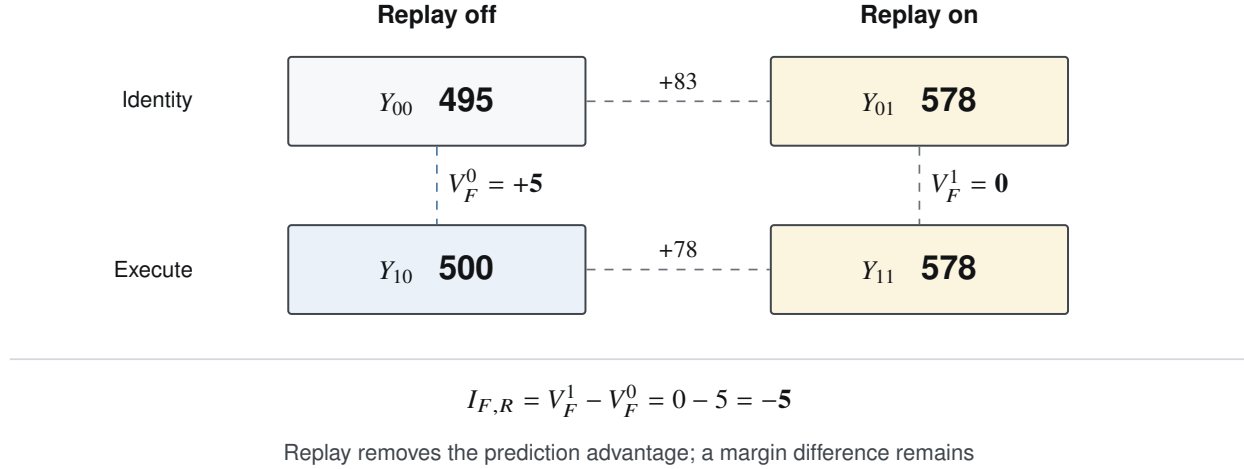

The four-cell result establishes continuation-dependent action value for a certified write; the full trajectory separately tests aggregate performance under repeated legal commits. The uncoordinated full-replay trajectory with Fiber makes three legal commits yet ties its historical no-Fiber reference at $950$ correct predictions. Legal actuation is therefore not, by itself, an improvement guarantee. A separate 75\%-retained nested path yields proposal values $0,-3,-1,-3$ whose exact telescoping sum is $-7$ (Appendix~\ref{app:support}); those values belong to that path and are not transferable to the full-replay parent.

\subsection{Targeted benefit and decision-level saturation}
The item-level comparison localizes all five no-replay correctness gains to label 11, the repaired class, with no correctness losses. In contrast, replay after identity increases the old-task count by $85$ and decreases the current-task count by $2$. Table~\ref{tab:taskeffects} decomposes these two conditional effects. It supports a targeted observed Fiber benefit and a broader observed replay benefit on this panel; it does not establish disjoint parameter support or a universal division of labor.

\begin{table}[!htbp]\centering
\caption{\textbf{Conditional effects by task at the 32-update endpoint.} Both contrasts use $Y_{00}$ as reference. Each task has 200 validation items. These are two edges of the same four-cell experiment, not independent experiments; all five Fiber gains are in label 11.}
\label{tab:taskeffects}
\begin{tabular}{lrr}\toprule
Evaluation task & Fiber without future replay & Replay without anchor Fiber\\
 & $Y_{10}-Y_{00}$ & $Y_{01}-Y_{00}$\\\midrule
Task 0 (old)&0&$+47$\\
Task 1 (old)&$+5$&$+38$\\
Task 2 (current)&0&$-2$\\\midrule
Total&$+5$&$+83$\\\bottomrule
\end{tabular}
\end{table}

\Needspace{5\baselineskip}
With replay, execute and identity agree on all 600 predicted labels, but their margins and CE differ. In particular, Table~\ref{tab:h32margin} shows that the repaired representative retains a nonzero margin difference. Identical predictions therefore indicate decision-level saturation at this endpoint, not disappearance of the internal perturbation.

\begin{table}[!htbp]\centering
\caption{Label-11 endpoint margin after 32 updates. The diagnostic is performed on an endpoint clone; it is not an additional training action.}
\label{tab:h32margin}
\begin{tabular}{lrrrc}\toprule
Future replay & Identity & Execute & Fiber increment & Would either commit?\\\midrule
Off & $-0.7705$ & $0.4308$ & $1.2013$ & Yes\\
On & $7.1279$ & $7.6531$ & $0.5252$ & No\\\bottomrule
\end{tabular}
\end{table}

The margin interaction is
\begin{equation}
 I_{m_{11}}=(7.6531-7.1279)-(0.4308-(-0.7705))\approx-0.6761.
\end{equation}
Thus the structural contrast is negative rather than a positive repair-stabilization effect. Replay improves the endpoint substantially even without the anchor repair and reduces its marginal contribution here. Neither a larger representative margin nor a successful certificate establishes positive aggregate future utility.

\subsection{Continuation effects at longer horizons}
The longer-horizon comparisons use parents from a \emph{75\%-retained} replay trajectory with the Fiber actuator. Replay-on follows the retained schedule, not full replay. Table~\ref{tab:interaction} reports the conditional values at the specified endpoints. Each row is a matched same-parent comparison; the contrasts were checked against stored numerical readouts offline, without rerunning training.

\begin{table}[!htbp]\centering
\caption{\textbf{Continuation-dependent effects at longer horizons.} Each row compares the same parent and proposal under two future replay conditions. $H$ counts subsequent ordinary updates. Evaluation populations and active classes grow across task ends.}
\label{tab:interaction}
\begin{tabular}{rrrcrrr}\toprule
Anchor & $H$ & Eval. step & Items/classes & $V_F^R$ & $V_F^0$ & $I_{F,R}$\\\midrule
511 &128&639&800/40&0&$-5$&$+5$\\
511 &288&799&1,000/50&$-3$&$-1$&$-2$\\
351 &128&479&600/30&$-1$&$+5$&$-6$\\
351 &288&639&800/40&$+1$&$+1$&0\\
351 &448&799&1,000/50&0&0&0\\\bottomrule
\end{tabular}
\end{table}

At anchor 351, $H=128$, changing replay changes the action-value sign. At anchor 511 the interaction sign changes across the evaluated endpoints. The reported object is $I_{F,R}(x,H;\omega,Q_H)$: the evidence does not isolate elapsed time from the changing readout. Nor can these values be joined to the full-replay 32-update experiment as a single ``step-351'' response curve. Parent history and exact proposal, not step number alone, identify the intervention. The 32-update comparison establishes nonzero interaction separately from this longer-horizon family.

\FloatBarrier

\subsection{Output distillation supplies a second continuation family}
\label{sec:distillationinteraction}
A separate four-cell experiment tests the interaction under an output-distillation loss. On the frozen LwF-plus-Fiber development trajectory, the first three legal nonzero commits occur at steps 191, 223, and 255. Each exact parent is independently forked into identity/execute and distillation-off/on branches. Replay examples are absent from the training loss in all four branches; current-task CE remains active, and subsequent Fiber writes are disabled. The teacher-refresh rule is common across branches, but teachers obtained from branch-local students may diverge at later task boundaries.

The LwF-style loss uses temperature $T=2$ and weight $\lambda=1$:
\begin{equation}
 \mathcal L=\operatorname{CE}(z_s,y)
 +\lambda T^2\sum_{h<k}\operatorname{KL}\!\left(
 \operatorname{softmax}(z_{\tau,h}/T)\,\Vert\,
 \operatorname{softmax}(z_{s,h}/T)\right).
 \label{eq:lwf}
\end{equation}
Here the CE uses all active logits and each KL is normalized within a previous ten-class head. At a task boundary, before head growth, the student is copied to a frozen teacher. Distillation-off omits the KL term but preserves the teacher state machine, head growth, and reservoir admission. Fiber still accesses historical representatives; this is not a memory-free method.

\begin{table}[!htbp]\centering
\caption{\textbf{Interaction beyond replay and the original backbone.} Every row is one activated parent, not an independent replication at each horizon. $n_{32}/n_{128}$ records the evaluation sizes. The distillation comparison uses one development root; the backbone comparison uses three RoBERTa roots, numbered locally. Complete branch counts and CE contrasts appear in Appendix~\ref{app:extensions}.}
\label{tab:continuation_backbone}
\begin{tabular}{llrrrr}\toprule
Continuation/backbone & Root & Anchor & $n_{32}/n_{128}$ & $I_J(32)$ & $I_J(128)$\\\midrule
Distillation / BERT & Development &191&400/400&$+3$&$+1$\\
Distillation / BERT & Development &223&400/600&$+2$&$-4$\\
Distillation / BERT & Development &255&400/600&$+5$&$-7$\\\midrule
Replay / RoBERTa & 1&511&800/800&$+2$&$+5$\\
Replay / RoBERTa & 2&511&800/800&$+1$&$+3$\\
Replay / RoBERTa & 3&575&800/1,000&$-3$&$+2$\\\bottomrule
\end{tabular}
\end{table}

All six distillation contrasts are nonzero on correct count (Table~\ref{tab:continuation_backbone}). This identifies continuation dependence under a second external loss, within one shared prefix history. Sign reversals at anchors 223 and 255 also change the active-class panel; they are not pure fixed-readout temporal reversals. Interaction signs can differ between correct count and $-\mathrm{CE}$, as the full table records. The complete-policy distillation comparison, $414\to433$, provides a separate application witness with task-level losses (Appendix~\ref{sec:lwf}); it is not one of these four-cell estimates.

\subsection{A second backbone}
\label{sec:backbone}
With RoBERTa-base and the unchanged target, cap, certificate, and consultation cadence, three new roots use the first legal nonzero Fiber event at or before step 671. The AdamW recipe and replay actuator are held fixed; tokenization and pretrained representation are backbone-specific. Roots~1 and~2 activate at 511, and Root~3 at 575. The corresponding correct-count interactions are $(+2,+1,-3)$ at 32 updates and $(+5,+3,+2)$ at 128 updates.

These are six nonzero contrasts across three roots, not evidence of a model-invariant sign or magnitude. At Root~1, $H=128$, for example, Fiber has value $-5$ without replay and $0$ with replay: the positive interaction $+5$ does not mean Fiber improves on the replay-only branch. Root~3 also uses a larger panel at the longer horizon. The supported claim is occurrence under this second backbone, not a controlled causal comparison of backbone architectures.

\subsection{Optimizer-native, phase-matched deployment}
\label{sec:optimizer}
An optimizer helps generate intrinsic learning dynamics. A cross-optimizer test can therefore hold the \emph{Fiber construction rule} fixed without requiring equality of realized states, directions, or absolute activation times. In the optimizer-native, phase-matched study, AdamW and momentum SGDW share each of three new root identifiers and the associated exogenous data, initialization, admission, and sampling rules. SGDW uses momentum $0.9$ and decoupled multiplicative decay, without inserting the decay term into its momentum buffer. Its learning rates are selected only by Task-0 validation performance after ordinary training from four frozen encoder/head pairs; $(0.003,0.3)$ is selected, with $199/200$ correct on the calibration validation task. This design does not match root-mean-square (RMS) update magnitudes to AdamW.

Two finite opportunity sets specify the phases:
\begin{equation}
 \begin{split}
 \mathcal E_2&=\{351,383,415,447,479\},\\
 \mathcal E_3&=\{511,543,575,607,639\}.
 \end{split}
 \label{eq:phase_sets}
\end{equation}
For optimizer $o$, root $r$, and phase $p$, select the first $t\in\mathcal E_p$ at which the unchanged rule produces a certified nonzero proposal on the unactuated scout. If none exists, report phase nonactivation. A selected parent is saved before the measurement write; all four branches reload it. The scout never retains a probe write and no later Fiber write occurs within a branch.

All six optimizer--root pairs are saturated at every Task-2 opportunity and therefore nonactivate in that phase. All six activate at the first Task-3 opportunity, 511. The rule \emph{permits} different activation times; the observed times happen to agree. Neither equality of times nor a saturated phase establishes equality of learning states or absence of useful interventions beyond the specified frame and target.

\begin{table}[!htbp]\centering
\caption{\textbf{Optimizer-native replay interaction at the activated Task-3 parents.} All rows use anchor 511 and the same 800-item, 40-class panel at both horizons. Task 2 has six phase nonactivations, excluded from this conditional-interaction table but not from the study denominator. Root numbers are local to this study; the two optimizer conditions are paired within each root. $I_{-\mathrm{CE}}$ uses larger-is-better negative CE.}
\label{tab:native_optimizer}
\begin{tabular}{llrrrr}\toprule
Root & Optimizer & $I_J(32)$ & $I_J(128)$ & $I_{-\mathrm{CE}}(32)$ & $I_{-\mathrm{CE}}(128)$\\\midrule
1 & AdamW &$-2$&$-6$&$-0.024153$&$-0.048234$\\
1 & SGDW &0&$-1$&$-0.105356$&$-0.118092$\\
2 & AdamW &$-6$&$-8$&$-0.017977$&$-0.043018$\\
2 & SGDW &$-2$&$-2$&$-0.024880$&$-0.029122$\\
3 & AdamW &0&0&$-0.009414$&$-0.009821$\\
3 & SGDW &0&$-3$&$-0.054441$&$-0.067756$\\\bottomrule
\end{tabular}
\end{table}

At $H=128$, all three activated SGDW roots have nonzero correct-count interaction, $(-1,-2,-3)$; two of three AdamW roots do, $(-6,-8,0)$. All twelve activated optimizer--root--horizon contrasts are negative on $-\mathrm{CE}$. The matched optimizer-native branches thus show interaction under SGDW trajectories as well as AdamW on this specified panel.

A zero count interaction does not establish equality of continuous readouts. The twelve contrasts are nested within three training roots, with optimizer conditions paired within each root and horizons treated as repeated readouts; no population-level significance test is inferred.

Negative interaction describes subadditivity on the specified utility scale. For AdamW Root~1 at $H=128$,
\begin{equation}
 (Y_{00},Y_{01},Y_{10},Y_{11})=(663,753,667,751),
 \qquad I_J=-6.
 \label{eq:native_optimizer_example}
\end{equation}
Fiber adds four correct predictions without replay and loses two with replay. Conversely, SGDW Root~1 at $H=32$ has $V_F^0=V_F^1=6$ and hence $I_J=0$, despite positive conditional Fiber effects. Interaction, benefit, and state change are distinct predicates.

\subsection{Fixed-anchor update-scale sensitivity}
\label{sec:scale_sensitivity}
The update-scale-matched sensitivity study asks whether interaction is also observed under a controlled scale criterion. It keeps anchors 351 and 511 and calibrates encoder and head learning rates independently to selected Task-0 update-RMS probes. The frozen channel-wise matching gate passes at rates $(0.01,0.02)$, with aggregate multiplicative mismatch factors $1.436$ and $1.198$. Four of six root--anchor cells activate. Ordered by Root~1 at 351 and 511, then Root~2 at 351 and 511, their count interactions are $(0,-4,0,0)$ at 32 updates and $(-1,+3,+1,0)$ at 128 updates. Root~3 is saturated at both registered anchors, not a zero-interaction observation. For Root~1 at anchor 511, the sign changes on the same 800-item panel, providing a fixed-readout horizon-dependent contrast.

The scale criterion matches measured update magnitudes, not full vector fields, moments, or trajectories. This sensitivity comparison and the native-optimizer comparison use different root panels and identify different conditions for the occurrence of interaction. Appendix~\ref{app:optimizer_diagnostics} also reports native-event and linked-rate fixed-anchor diagnostics on the sensitivity panel, including their nonactivations. Those diagnostics use distinct recipes on shared roots; differences between them do not isolate a calibration mechanism or supply independent replications. Together, the distinct native phase-matched and scale-matched designs establish scoped occurrences of interaction under SGDW.

\section{Closed-loop coordination under a replay-workload constraint}
\label{sec:coord}
Having identified continuation-conditioned interaction, we ask whether a fixed coordinator can use a structural signal to improve subsequent replay allocation. All replay-allocation policies in this section use the same period-32 Fiber interface unless an explicit shadow/no-write control is specified. The central comparison is therefore coordination \emph{given Fiber}, not Fiber versus no Fiber. A complete coordinator changes future states, writes, and allocations; its effect is the policy contrast in Equation~\eqref{eq:policy}.

The development reference is full replay with Fiber. Its endpoint and the original aligned-policy endpoint are archived results, reused rather than independently rerun for each attribution control. The development comparisons examine reinforcement, withdrawal, and deprioritization after a repair. The attribution controls isolate aspects of the allocation rule, and the five-root test evaluates two prespecified policy contrasts on an unused test split.

\subsection{A bounded signal-guided allocation rule}
The reinforcement policy retains full replay and uses the pre-action signal accompanying each accepted Fiber repair. For the next 32 updates, it begins with the standard replay candidate and changes at most one row per update. A candidate non-repaired class with the highest available pre-action representative margin donates a position to a repaired class in the \emph{same source task}. A deterministic keyed draw selects an unused reservoir item. If no eligible exchange exists, the original candidate is retained.

Within each active window, the implementation requires a unique-UID candidate and keeps the replacement batch unique. It preserves replay batch size and source-task identity at each replay position. Outside these windows the baseline candidate is unchanged; no global UID-uniqueness claim is made. The margin signal is held fixed over the window; validation outcomes do not enter execution. The donor margin is a fixed heuristic, not a safety certificate. The rule is therefore a concrete workload-constrained coordinator, not an optimal allocation algorithm.

\subsection{Fixed-workload outcomes and the limits of support withdrawal}
The original Fiber-aligned policy reaches $953$ rather than $950$ correct predictions, with mean CE $0.2130189$ rather than $0.2195535$. Both use 799 replay events and 12,784 replay examples. The policy redirects 96 rows (about 0.751\%) in windows 352--383, 512--543, and 672--703, following commits at 351, 511, and 671. Terminal task differences are $(+1,+1,+1,0,0)$, comprising seven corrected items and four new errors. Task-end differences $(0,0,+1,-1,+3)$ show that this terminal gain is not trajectory-wide non-degradation.

\begin{figure}[!htbp]
\centering
\resizebox{0.96\linewidth}{!}{
\begin{tikzpicture}[x=1cm,y=1cm]
  \path[use as bounding box] (0,0) rectangle (14.90,8.11);
  \node[ffpanel] at (0,7.90) {COORDINATED REPLAY: POLICY OUTCOMES AND ATTRIBUTION};
  \node[fflabel,anchor=west,font=\sffamily\scriptsize\bfseries] at (0,7.25) {Policy};
  \node[fflabel,font=\sffamily\scriptsize\bfseries] at (5.60,7.25) {Terminal count};
  \node[fflabel,font=\sffamily\scriptsize\bfseries] at (8.05,7.25) {$J$};
  \node[fflabel,font=\sffamily\scriptsize\bfseries] at (9.03,7.25) {$\Delta J$};
  \node[fflabel,font=\sffamily\scriptsize\bfseries] at (10.22,7.25) {Writes};
  \node[fflabel,font=\sffamily\scriptsize\bfseries] at (11.57,7.25) {Events};
  \node[fflabel,font=\sffamily\scriptsize\bfseries] at (13.67,7.25) {Examples};
  \draw[ffrule] (0,6.93)--(14.90,6.93);
  \foreach \v in {0,250,500,750,1000}{
    \pgfmathsetmacro{\xx}{3.75+3.72*\v/1000}
    \draw[ffrule] (\xx,1.54)--(\xx,6.72);
    \node[fflabel,anchor=north] at (\xx,1.41) {\ifnum\v=1000\relax1,000\else\v\fi};
  }
  \draw[ffline] (3.75,1.54)--(7.47,1.54);
  \foreach \yy/\name/\j/\delta/\writes/\events/\examples in {
    6.55/Full replay + Fiber/950/{0}/3/799/{12{,}784},
    5.84/Replay withdrawal/851/{-99}/10/511/{8{,}176},
    5.13/Replay deprioritization/948/{-2}/10/799/{12{,}784},
    4.42/Guided reinforcement/953/{+3}/3/799/{12{,}784},
    3.33/Mask-matched random/953/{+3}/3/799/{12{,}784},
    2.62/Permuted signal/955/{+5}/3/799/{12{,}784},
    1.91/Shadow signal; no write/952/{+2}/0/799/{12{,}784}}{
      \node[fflabel,anchor=west] at (0,\yy) {\name};
      \node[ffmath] at (8.05,\yy) {$\j$};
      \node[ffmath] at (9.03,\yy) {$\delta$};
      \node[ffmath] at (10.22,\yy) {$\writes$};
      \node[ffmath] at (11.57,\yy) {$\events$};
      \node[ffmath] at (13.67,\yy) {$\examples$};
      \pgfmathsetmacro{\xx}{3.75+3.72*\j/1000}
      \draw[ffrule] (3.75,\yy)--(\xx,\yy);
      \node[circle,draw=FFGrayDark,fill=white,line width=.8pt,minimum size=5.1pt,inner sep=0pt] at (\xx,\yy) {};
  }
  \draw[ffdashline] (0,3.94)--(14.90,3.94);
  \node[fflabel,anchor=west,fill=white,text=FFGrayDark,inner xsep=2pt] at (.12,3.94) {Registered attribution controls};
  \node[fflabel,text=FFGrayDark] at (7.45,.81) {Same replay workload except withdrawal; writes are retained parameter changes};
  \node[fflabel,text=FFGrayDark] at (7.45,.32) {Random and permuted controls retain Fiber-triggered windows; not independent-root replications};
\end{tikzpicture}}
\caption{\textbf{Policy improvement does not identify a privileged allocation rule.} The full 0--1,000 count scale and exact numbers show the development coordination policies and the registered attribution controls. $\Delta J$ is relative to the archived full-replay-plus-Fiber reference. ``Writes'' counts retained head changes: the shadow arm accepts three proposals but rolls all three back. The other arms retain their accepted writes under the same period-32 consultation rule, not an imposed common write sequence. All rows except withdrawal match replay event and example counts. The random and permuted controls match or exceed the original aligned endpoint; this is one development-root comparison, not an uncertainty estimate or total-compute match.}
\label{fig:policies}
\end{figure}

Post-repair replay withdrawal disables replay for 32 updates after each accepted repair. Further accepted repairs extend suppression, producing blocks 352--511 and 672--799, ten commits, and terminal $J=851$. At step 479, the task-wise deficit is $(-51,-27,0)$: the loss lies in older tasks, not the current one. This rejects the implemented renewable rule on this root, not all bounded relief rules. Repair-conditioned replay deprioritization preserves all replay counts but redirects repaired-class occurrences to other classes in the same source task. It changes 112 rows, produces ten commits, and reaches $948$. Together, the two outcomes preclude treating a local repair certificate as an automatic certificate of external-support redundancy.

\subsection{What the allocation controls identify}
\label{sec:attribution}
Three registered controls retain the model, data stream, root, period-32 Fiber consultation rule, replay counts, and per-position source-task constraint. Their accepted proposals and future windows remain endogenous to each trajectory. In the observed runs, each control also makes 96 replacements in the same three windows. This last equality is a measured outcome, not a protocol that forces the historical 96-position mask or identical write tensors.

\emph{Mask-matched random} retains actual Fiber writes and uses the true pre-action signal to compute whether the aligned rule would replace a row on the control's own trajectory. At each such step, a keyed random rule selects a feasible donor position and an unused reservoir item from the same source task, without using class labels or margins for that content choice. It reaches $953$. This is a content-selection control conditional on a signal-derived change mask, not an entirely unguided replay policy.

\emph{Permuted signal} also retains the writes. Before replay selection, it relabels the signal within each ten-class task by $j\mapsto(j+3)\bmod10$, carrying the margin values with the relabeling, and applies the unchanged allocation rule. It reaches $955$. Matching or exceeding $953$ under these altered content rules means that exact Fiber-label alignment is not necessary to reach the original endpoint in this comparison. It does not show that allocation content has no effect, that permutation is generally better, or that all Fiber-derived information is dispensable.

\emph{Shadow signal/no write} computes and certifies each proposal, restores the exact pre-proposal head and commit/token state before the next ordinary update, and uses the true signal for replay allocation. It accepts three proposals, retains zero writes, and reaches $952$. The $953-952=1$ gap is a complete-policy contrast with endogenous future signals and allocations, not a controlled direct effect holding the whole replay path fixed. The shadow result shows that a positive shift over the $950$ reference is possible without retained writes; it establishes neither their necessity for improvement nor their irrelevance.

The separate single-window reinforcement comparison gives $577$ rather than $578$ at $H=32$ despite increasing the repaired margin. The full policies include later actions and allocation windows, so their terminal differences are not pure horizon extensions of that one-shot $-1$. Local repair, measured interaction, and complete-policy performance therefore require separate comparisons.

\subsection{Allocation value across training roots}
\label{sec:policy_replication}
The five-root policy study starts each of five fresh roots from update 0 under three frozen complete policies: standard full replay plus period-32 Fiber ($S$), Fiber-guided reinforcement plus period-32 Fiber ($G$), and mask-matched random content plus period-32 Fiber ($M$). Thus $S$ is not a no-Fiber baseline. The two primary estimands are $\Delta_{\rm policy}=J_G-J_S$ for the additional value of guided replay allocation and $\Delta_{\rm specificity}=J_G-J_M$ for Fiber-signal content specificity, conditional on the same Fiber construction and consultation cadence in every arm. The random arm retains its own signal-derived change mask; it is not a signal-free policy. Every arm keeps 799 replay events and 12,784 replay presentations.

The frozen policy leaves a native candidate unchanged when duplicate UIDs violate the replacement primitive's uniqueness precondition. This preserves native replay semantics; Appendix~\ref{app:policy_details} specifies this duplicate-candidate rule.

\begin{table}[!htbp]\centering
\caption{\textbf{Fresh-test primary outcomes of the five-root policy study.} All counts are out of the same 1,500 items. $S$ is standard full replay plus period-32 Fiber, $G$ is guided reinforcement, and $M$ is mask-matched random content. Each row is one paired training root. Both primary contrast vectors are mixed or inconclusive under the frozen directional criterion.}
\label{tab:policy_primary}
\begin{tabular}{lrrrrrr}\toprule
Root & $J_S$ & $J_G$ & $J_M$ & $G-S$ & $G-M$ & $M-S$\\\midrule
1&1,343&1,344&1,341&$+1$&$+3$&$-2$\\
2&1,351&1,348&1,350&$-3$&$-2$&$-1$\\
3&1,323&1,327&1,324&$+4$&$+3$&$+1$\\
4&1,363&1,360&1,363&$-3$&$-3$&0\\
5&1,346&1,344&1,347&$-2$&$-3$&$+1$\\\midrule
Mean difference&&&&$-0.6$&$-0.4$&$-0.2$\\
Median difference&&&&$-2$&$-2$&0\\\bottomrule
\end{tabular}
\end{table}

Guided-minus-standard is $(+1,-3,+4,-3,-2)$ and guided-minus-random is $(+3,-2,+3,-3,-3)$. Each is positive in two roots and negative in three. Both are \emph{mixed or inconclusive} under the frozen rule. The mean accuracy differences are $-0.040$ and approximately $-0.027$ percentage points, respectively. The primary correct-count effects are therefore root-dependent under the frozen criterion, with no root-stable positive advantage for guided allocation or for exact Fiber-label alignment over the matched-random content control. This is not proof of exact equivalence or a negative population effect.

\paragraph{Secondary continuous readout.}
Guided-minus-standard test CE is negative in all five roots, with vector
\begin{equation}
 (-0.0108714,\,-0.0004370,\,-0.0005938,\,-0.0005780,\,-0.0005924).
 \label{eq:policy_ce}
\end{equation}
Its mean is $-0.0026145$ and median $-0.0005924$; the mean is strongly influenced by Root~1. Guided-minus-random CE is negative in four roots, with mean $-0.0020247$ and median $-0.0017579$. Lower mean CE indicates a better aggregate log-loss on this panel, not improved calibration or margins for every item. It is compatible with a worse correct count. The primary count verdict is unchanged. Appendix~\ref{app:policy_details} reports all CE values and the reused-validation comparisons; validation alone favors random content directionally, whereas its fresh-test count contrast is mixed.

\paragraph{The external rule changes later intrinsic decisions.}
The comparison also makes endogenous feedback observable. On Root~1, standard and random policies commit at $(511,543,671)$, while guided commits at $(511,671)$. On Root~3, standard and random commit only at 511, while guided also commits at 671. After the first allocation change the trajectories are no longer a common parent; subsequent geometry and activation can differ. These are closed-loop policy outcomes, not additional matched-parent interaction coefficients or proof that any changed event mediates the terminal score. They show why holding the Fiber construction and period-32 opportunities fixed does not hold realized writes fixed.

The three-arm comparison does not test Fiber removal, retained-write necessity, or the causal value of trigger-time selection; its secondary CE result cannot supply those comparisons. The development controls and fresh-root comparison therefore show that identifying intrinsic--extrinsic interaction does not by itself determine a uniformly beneficial fixed allocation rule. The five-root correct-count effects instead characterize the root-dependent value of that rule with Fiber present in every arm.

\FloatBarrier
\section{Repair geometry, functional readouts, and sufficient-quality control}
\label{sec:geometryinterpretation}
\subsection{Three spaces that must remain distinct}
Corollary~\ref{cor:repairset} characterizes an ideal \emph{parameter-space} set on an immutable finite frame. A different question concerns favorable \emph{terminal outputs}. A third asks whether a low-cost rule identifies a \emph{training-reachable state or control region} with favorable future value. The first two can be analyzed with the stored data; they do not answer the third.

For an output matrix $L$ and a fixed set of items $S$ that are strictly correctly classified, define
\begin{equation}
 \mathcal C_S=\{L:L_{i,y_i}-L_{i,k}>0\ \text{for every }i\in S,\ k\ne y_i\}.
 \label{eq:outputcell}
\end{equation}
It is an open convex polyhedral region and has $J(L)\ge |S|$. Hence two output matrices that strictly agree on $|S|$ correct items preserve those items along their entire logit segment. Away from top-logit ties, the superlevel set $\{L:J(L)\ge k\}$ is a union of the regions $\mathcal C_S$ over size-$k$ item sets $S$ and need not be convex. At ties, classification follows the declared lowest-index argmax convention. This is a property of the readout, not a Fiber-specific theorem.

The development random and permuted controls share 951 correct items, already certifying a segment with $J\ge951>950$. Exact margin-boundary calculations on their stored FP32 logits sharpen this to a minimum of 953. The random--shadow and permuted--shadow segments each have minimum 952. These three segments form a favorable connected subset \emph{of output space}; CE convexity also places their entire logit convex hull below the reference CE. Appendix~\ref{app:outputgeometry} gives the proof, boundary convention, and numerical certificates. These are algebraic analyses of saved outputs, not newly trained ensembles, parameter interpolation experiments, or realizable replay-mixture trajectories. No whole-triangle count guarantee is inferred from its boundary.

\subsection{Similar functional quality does not require state convergence}
The final development reallocation window ends at step 703. Thereafter the random and permuted arms use the same ordinary exogenous sequence and have no further nonzero Fiber action. Nevertheless, the stored head distance $\norm{\Theta_M-\Theta_P}_F$ is $0.24563$, $0.25140$, $0.25227$, and $0.25538$ at steps 703, 735, 767, and 799. Similar terminal counts, 953 and 955, therefore cannot be used as evidence that head differences disappeared or contracted monotonically.

Correct count records whether margins cross zero; mean CE records an aggregate continuous loss. Either can be insensitive to some state differences without the state differences vanishing. The five-root allocation comparison and the native-optimizer interaction study likewise distinguish count from CE outcomes. That common distinction does not identify a shared cause. A head-coordinate comparison also does not rule out contraction of other components or attenuation of a future-readout projection.

\subsection{Local feasibility is not future-value maximization}
The minimum-norm write solves a prescribed present-frame problem. It need not maximize future accuracy, and a commit does not require subsequently solving that harder problem. On a differentiable fixed branch, the first-order sensitivity of a smooth terminal utility $S(x_T)$ to a continuous control relaxation at $t$ is
\begin{equation}
 D_{u_t}S(x_T)=B_t^\top\Phi_{T,t+1}^\top\nabla S(x_T),
 \quad A_t=D_xF_t,\quad B_t=D_uF_t,\quad
 \Phi_{T,t+1}=A_{T-1}\cdots A_{t+1}.
 \label{eq:future_sensitivity}
\end{equation}
The chain rule through training \cite{hypergrad} includes optimizer state, subsequent updates, and the final readout. It is not in general the current representative-margin gradient. The discrete replay choices and switching acceptance rule require additional assumptions before such a derivative can describe them. No future-gradient alignment, repair cone, or contraction coefficient is estimated here. Appendix~\ref{app:valueboundaries} derives the conditional sensitivity result; Appendix~\ref{app:policy_difference} gives an exact finite-policy decomposition without a differentiability assumption.

A legitimate engineering objective is to minimize selection cost subject to a required quality level, rather than maximize unpriced terminal utility:
\begin{equation}
 \min_{\Pi}C_{\rm selection}(\Pi)
 \quad\text{subject to}\quad V(\Pi)\ge v_{\rm required}.
 \label{eq:sufficient}
\end{equation}
This states a design objective, not a claim that the present coordinator solves it or meets a universal quality threshold. The rule uses no online future-rollout search and no learned outcome selector. The three development attribution arms each log 20 extra encoder calls over 648 frame rows and zero extra backward passes for Fiber consultation. These component counts do not include all host computation, certification, or wall time. They support a bounded structural interface, not negligible overhead, equal total compute, or optimal cost-efficiency.

Multiple favorable realizations do not imply multiple global optima, and a small gap among tested policies does not bound the best untested policy. Conversely, a useful engineering rule need not establish global optimality. The evidence does not justify avoiding search on the assumption that continued training necessarily erases its benefits. The present rule supplies a bounded structural intervention without online future-rollout search while leaving priced future-value selection as a distinct objective.

\section{Related work and scope of the contribution}
\label{sec:related}
\paragraph{State, observation, and sequential value.}
Classical nonlinear control distinguishes state from output and studies distinguishability through future input--output behavior \cite{hk}. Dynamic programming and reinforcement learning make value conditional on subsequent decisions \cite{bellman,sutton}. These established distinctions provide the starting point for the present formulation. The present fiber is an instantaneous level set of a declared observation; its directions need not remain invisible after learning. The head transaction is a constrained state reset. We neither prove controllability through training inputs nor infer a closed low-dimensional evolution from an observation frame.

Fiber Fingerprints specializes present/future distinctions to complete learning execution states and controlled future training \cite{p1}. Revelation Control already studies priced intervention choice and separates information value from productive reuse \cite{p2}. The present paper connects observation-relative intrinsic actuation to readout-specific intrinsic--extrinsic coupling: a finite write realizes the internal interface, matched replay and distillation continuations identify readout-specific interaction, while separate policy tests examine coordination. Function-preserving ReLU rescaling also changes optimization dynamics \cite{pathconditioned}; present agreement with different future learning is therefore not sufficient to distinguish the present contribution.

\paragraph{Protected updates and model editing.}
The least-norm construction is an application of generalized-inverse linear algebra \cite{penrose}. In continual learning, OGD projects gradients to protect previous outputs, and GPM uses activation subspaces to restrict updates \cite{ogd,gpm}. In model editing, ROME and MEMIT directly modify selected parameter associations \cite{rome,memit}. AlphaEdit is the closest algebraic comparator: its update is projected into a null space associated with preserved knowledge \cite{alphaedit}. We do not claim null-space protection as new. Our instantiated constraints instead protect a fixed \emph{current} frame while changing selected \emph{historical} margins, and explicitly certify the rounded write. This specifies a controlled actuation interface, not general superiority over those editing algorithms.

\paragraph{What later training does to an edit.}
Wen and Zhang directly study how downstream fine-tuning affects edited knowledge and compare editing methods and fine-tuning objectives \cite{retention}. Our evidence is organized around a different estimand: the same pre-action parent and proposal are evaluated in all four execute/identity and replay-on/off branches. This identifies a replay-conditioned difference in \emph{marginal endpoint utility}, including prediction-level saturation with a remaining margin difference. We then test support withdrawal and count-matched coordination. The distinction is the specified constrained-write/factorial-continuation comparison and its policy-level follow-up, not a claim that interactions between editing and training were previously unknown.

\paragraph{Continual learning and replay allocation.}
Experience replay has a long history \cite{lin} and remains effective with small episodic memories \cite{er}. LwF and EWC constrain subsequent learning through distillation or parameter regularization \cite{lwf,ewc}; GEM and A-GEM use episodic information to constrain updates \cite{gem,agem}. These methods are not simply ``external'' in an ontological sense: they also depend on learned internal quantities. Here, the terminology distinguishes a direct state intervention from a subsequent training rule.

State-dependent replay selection also predates this work. MIR selects memories by predicted interference under a candidate update \cite{mir}. Our bounded coordinator uses a fixed Fiber-derived signal and preserves replay count and source-task positions; it is not proposed as the first adaptive replay method or evaluated against MIR, GEM, or model-editing baselines. The withdrawal result tests whether \emph{this} local certificate licenses reduced support. The registered controls and five-root study in Section~\ref{sec:coord} evaluate extra replay-allocation value conditional on period-32 Fiber; they do not identify uniquely valuable semantic guidance or a necessary retained write. Self-Interventional Learning instead studies structural perturbations and predictive self-models used to guide actions \cite{sil}; our policy does not train such a self-model.

\paragraph{Factorial interaction versus dynamical mechanism.}
Potential outcomes and factorial interactions are established causal tools \cite{rubin,factorial}. Equation~\eqref{eq:interaction} is their unnormalized mixed contrast, not a new interaction statistic. All four computational branches are observed for a fixed root, but repeated evaluation items are not independent training-root replications. The design is not an observational time-series difference-in-differences study, and no parallel-trends argument is used. It identifies the implemented replay treatment, including its batch-composition, mean-loss-normalization, and randomization consequences.

Learning-order analysis can study operator-level interactions through modified dynamics or gradient-field commutators \cite{dherin}. A scalar factorial contrast is weaker: nonlinear readout alone can create non-additivity even when state-update maps commute, as Appendix~\ref{app:contrast} demonstrates. Accordingly, ``coupling'' here denotes the identified continuation dependence of intervention effects, not a recovered physical interaction term or a universal dynamical law.

\FloatBarrier

\FloatBarrier
\section{Discussion and limitations}
\label{sec:discussion}
\subsection{What transfers across learning contexts}
The finite-frame interface turns a declared observation constraint into a certified state change. Matched four-cell branches then establish that the functional value of that change can depend on the external continuation. Occurrences under replay, output distillation, a RoBERTa backbone, and SGDW extend the observation beyond the initial BERT/AdamW replay example. Across these settings, the observation constraint and action rule remain explicit, while each experiment measures value under its specified continuation. The realized parent, direction, activation rate, sign, and magnitude need not be invariant.

Native, phase-matched deployment respects the trajectory generated by each optimizer; update-scale matching asks a complementary sensitivity question under a measured scale constraint. Both address readout-specific continuation dependence under their respective protocols. Neither matches complete vector fields or identifies a universal dynamical law.

\subsection{Interaction is not a policy guarantee}
The negative interactions in the native-optimizer study show that replay can reduce an intrinsic write's incremental value even when replay itself is useful. A zero interaction, conversely, can accompany positive conditional effects. Coupling therefore denotes non-additivity on the specified readout, rather than positive synergy.

A complete coordinator introduces additional feedback and requires separate evaluation. With Fiber in every arm, the correct-count effect of guided allocation varies across the five fresh roots rather than remaining uniformly beneficial; exact label alignment likewise shows no root-stable advantage over the mask-matched random content control. The lower aggregate CE against standard replay in all five roots remains secondary evidence. The complete-policy test therefore measures the root-dependent value of a fixed allocation rule at a different evidentiary level from matched interaction identification.

The attribution controls preserve different components of Fiber information: the random-content policy keeps a signal-derived change mask, permutation keeps signal-triggered windows, and the shadow policy keeps sensing while withholding the write. Consequently, they neither show that all Fiber information is dispensable nor establish a unique value for trigger timing or the necessity of a retained write. Later activations and allocations are endogenous after trajectories diverge. The affine repair set and favorable output paths specify useful mathematical objects, but neither identifies a training-reachable long-horizon repair basin or the mediator of each policy outcome. Similar terminal quality does not require convergence of the underlying states.

\subsection{Scope, resources, and remaining questions}
All studies share one selected CLINC class partition, task order, and small-memory regime. Backbone extension covers BERT-base-cased and RoBERTa-base; optimizer extension covers particular AdamW and momentum-SGD recipes. Mechanism measurements use reused validation panels, selected legal activations, and few roots. Growing active-class panels limit temporal comparisons unless a common readout is verified. The five-root policy study adds five fresh training roots and an unused test split, not new tasks or pretrained models. All root-wise primary contrasts and phase nonactivations are reported. Repeated horizons and optimizer diagnostics on shared roots are not independent confirmations, and the descriptive direction criteria are not population significance tests.

The transaction protects a finite inference frame and specified state coordinates, rather than all inputs or training-mode gradients. Its frame, its margin-deficiency signal, and a future-response fingerprint are distinct objects. Historical representatives remain accessible in the distillation experiments even though the ordinary loss has no replay. Equal replay or teacher-forward workload is not equal total compute: sensing, calibration, cloning, certification, and host-side selection have separate costs. A small changed-slot fraction does not bound a gradient or state perturbation. Offline numerical verification checked stored results and algebraic certificates; it did not reconstruct complete checkpoints or full GPU trajectories.

The resulting engineering question is how much future benefit can be obtained from an admissible structural intervention at an acceptable selection cost. The present rule needs no online future-rollout search, but the experiments do not identify a global optimum, a universal quality threshold, or a formal total-cost advantage. Progress on those stronger objectives would require evidence about the future-value information carried by the structural signal, beyond the local feasibility and conditional interaction established here.

\Needspace{12\baselineskip}
\section{Conclusion}
Intrinsic--extrinsic coupling concerns how direct learning-state intervention and subsequent learning jointly determine functional outcomes. Observation-relative fibers provide one precise description and actuation interface. Their finite-frame realization protects current observations while repairing specified historical margins. The local certificate establishes admissibility, while matched branches measure continuation-conditioned future value.

Matched experiments identify readout-specific four-cell non-additivity under replay and distillation, with further occurrences under a second backbone and optimizer-native SGDW dynamics. Complete-policy tests provide a complementary boundary: the identified coupling does not imply a universally beneficial fixed replay-allocation rule; correct-count effects remain root-dependent across five fresh training roots. Together, these results connect executable state geometry, conditional future value, matched interaction identification, and coordination while preserving the distinctions among them. Geometry defines admissible action; continuation and readout determine value; policy performance requires its own evidence.

\FloatBarrier
\clearpage
\begingroup
\small
\setlength{\parskip}{0pt}

\endgroup

\FloatBarrier
\clearpage
\appendix
\section{Mathematical guarantees and interpretation boundaries}
\label{app:proof}
We first prove the finite-frame interpolation, norm, and tolerance results, then distinguish scalar interaction, future-value sensitivity, and complete-policy effects.

\subsection{Projection, feasibility, and the least-norm property}
Let $U=\Delta\Theta^\top\in\R^{d\times q}$. Since $C$ has full row rank,
\begin{equation}
 P_C^\top=P_C,\quad P_C^2=P_C,\quad CP_C=0,
 \quad \operatorname{range}P_C=\ker C.
\end{equation}
Thus $CU=0$ is equivalent to $U=P_CU$, and the remaining constraint is $NU=D$, with $N=\mathsf H P_C$. Each row of $N$ is the orthogonal residual of one historical row outside the current row space. Consequently,
\begin{equation}
 \operatorname{rank}\!\begin{bmatrix}C\\\mathsf H\end{bmatrix}
 =\operatorname{rank}C+\operatorname{rank}N.
\end{equation}
If $N$ has full row rank, set
\begin{equation}
 U_*=N^\top(NN^\top)^{-1}D.
\end{equation}
Because $CN^\top=0$ and $\mathsf H N^\top=\mathsf H P_C\mathsf H^\top=NN^\top$, this satisfies $CU_*=0$ and $\mathsf H U_*=D$.

For any other feasible solution $U=U_*+Z$, one has $CZ=0$ and $NZ=0$. Every column of $U_*$ lies in $\operatorname{range}N^\top$, orthogonal to $\ker N$. Therefore
\begin{equation}
 \norm{U}_F^2=\norm{U_*}_F^2+\norm{Z}_F^2.
\end{equation}
This proves uniqueness of the least-norm solution. If $D\mathbf1=0$, then $U_*\mathbf1=0$, establishing the centering claim. Each head row separately lies in $\ker C$, so the action subspace dimension is $q(d-\operatorname{rank}C)$.

Even without full row rank of $N$, the constraints are feasible if and only if every column of $D$ lies in $\operatorname{range}N$. Necessity follows from $NU=D$. For sufficiency, take any $U_0$ with $NU_0=D$ and use $U=P_CU_0$; since $NP_C=N$, this satisfies both constraints. This argument does not prescribe the numerical implementation for singular frames; the implementation has no singular-solve fallback.

\subsection{Geometry-conditioned cost}
\label{app:costproof}
Put $G_N=NN^\top$. With $G_N$ symmetric positive definite,
\begin{equation}
 \begin{aligned}
 \norm{U_*}_F^2
 &=\operatorname{tr}\!\left(D^\top G_N^{-1}NN^\top G_N^{-1}D\right)\\
 &=\operatorname{tr}\!\left(D^\top G_N^{-1}D\right).
 \end{aligned}
\end{equation}
Diagonalizing $G_N$ in its orthonormal eigenvectors $\ell_i$, with eigenvalues $\sigma_i(N)^2$, gives Equation~\eqref{eq:geocost}. Bounding the reciprocal eigenvalues between $\sigma_{\max}^{-2}$ and $\sigma_{\min}^{-2}$ gives Equation~\eqref{eq:costbounds}. Orthogonality in the previous subsection shows why an exact ideal minimum above a budget excludes all exact solutions of that same target within the budget.

A small illustrative construction separates target size from geometry. Take $C=(1,0,1)$, $\mathsf H=(1,\epsilon,1)$, with $\epsilon>0$, and a fixed one-row target $D$. Then $N=(0,\epsilon,0)$ and
\begin{equation}
 \norm{\Delta\Theta_*}_F^2=\norm{D}_F^2/\epsilon^2.
\end{equation}
The requested logit changes have not grown as $\epsilon$ decreases; their realization becomes expensive because the historical feature approaches the protected current span. This is an algebraic example, not a measurement of the late LwF rejections.

\subsection{Margins, tolerances, and saturation}
\label{app:tolerance}
For every incorrect class $j$, the ideal target obeys $D_{i,y_i}-D_{i,j}=d_i$. The new pairwise margin is the original one plus $d_i$, proving Proposition~\ref{prop:leastnorm}'s margin claim. For the rounded write define
\begin{equation}
 E=\mathsf H\widehat{\Delta\Theta}^{\top}-D.
\end{equation}
Every pairwise post-margin differs from its ideal value by $E_{i,y_i}-E_{i,j}$, of absolute value at most $2\varepsilon_H$. Taking minima preserves this bound, yielding Equation~\eqref{eq:marginerror}. Since $\eta=2^{-10}=4\varepsilon_H$, a deficient row has lower bound $\mu+\eta-2\varepsilon_H=1+2^{-11}$. An initially nondeficient row does not have this overshoot unless already separated from the threshold. This is why interpolation tolerance does not remove the need for the explicit all-representative post-margin check.

For current logits $z$ and perturbation $e$ with $\norm{e}_\infty\le\varepsilon_C$, the log-sum-exp obeys
\begin{equation}
 e^{-\varepsilon_C}\sum_j e^{z_j}
 \le\sum_j e^{z_j+e_j}
 \le e^{\varepsilon_C}\sum_j e^{z_j}.
\end{equation}
Consequently its logarithm changes by at most $\varepsilon_C$. Since
$\operatorname{CE}(z,y)=\log\sum_j e^{z_j}-z_y$, the per-item CE change is at most $2\varepsilon_C$, as is the absolute mean change on the same panel. These are exact-real-arithmetic statements about the specified logits. The implementation independently checks native FP32 staged/live values and predictions.

After ideal repair all represented margins are at least $\mu$, so recomputation of Equation~\eqref{eq:target} gives $D=0$. The same argument applies to a rounded write accepted by the exact stored-frame all-margin postcheck. Reacquiring a frame after learning is a different operation, so no future-invariance conclusion follows.

\subsection{Preserved observations need not preserve backpropagation}
\label{app:gradient}
For the differentiable scalar loss $\ell(Z)$ with $Z=C\Theta^\top$, let $G=\nabla_Z\ell\in\R^{n_c\times q}$. The matrix chain rule gives
\begin{equation}
 \nabla_\Theta\ell=G^\top C,\qquad \nabla_C\ell=G\Theta.
\end{equation}
An exact write satisfying $C\Delta\Theta^\top=0$ keeps $Z$, $\ell$, and $G$ unchanged on the identical frozen frame. Its loss gradient with respect to the head is also unchanged in this particular calculation, but its derivative with respect to features changes by $G\Delta\Theta$. If $C=C_\psi$ comes from an encoder,
\begin{equation}
 \Delta(\nabla_\psi\ell)=(D_\psi C_\psi)^*\!\left[G\Delta\Theta\right],
\end{equation}
where the bias-one feature coordinate has zero encoder derivative. The null-space condition does not force this quantity to zero. Weight decay or subsequent different inputs provide further reasons not to infer an identical optimization step. The implemented frame is inference-mode and only toleranced after rounding, so these identities do not assert equality of the actual next training-mode head gradients either.

This calculation does not predict the sign of future utility or identify which path produced an empirical effect. It establishes that current-frame output preservation is compatible with changed future learning.

\subsection{Interaction is a contrast, not an operator commutator}
\label{app:contrast}
Proposition~\ref{prop:bilateral} concerns four scalar outcomes. It does not require differentiability. To see why nonzero interaction alone cannot establish noncommuting state updates, consider two scalar translations $z\mapsto z+f$ and $z\mapsto z+a$. They commute and give the additive endpoint $z+f+a$. For utility $J(z)=z^2$,
\begin{equation}
 J(z+f+a)-J(z+f)-J(z+a)+J(z)=2fa,
\end{equation}
which is nonzero when $fa\ne0$. This illustrative example is not a model of the neural experiments. It prevents a stronger operator-level interpretation from being attached to an outcome-level contrast without additional evidence. On an affine utility rescaling $J'=b+cJ$, the contrast becomes $I'=cI$; a nonlinear rescaling has no such general preservation rule.

\subsection{Future-value sensitivity and current margins}
\label{app:valueboundaries}
On a fixed-dimensional, differentiable branch, let $x_{t+1}=F_t(x_t,u_t)$ include every continuous model, optimizer, and controller coordinate needed for the subsequent rule. For a continuous relaxation of $u_t$, the chain rule gives Equation~\eqref{eq:future_sensitivity}. For multiple perturbations, under twice continuous differentiability on a neighborhood containing the perturbed branch,
\begin{equation}
 \delta S=\sum_t\langle s_t,\delta u_t\rangle+O(\norm{\delta u}^2),
 \qquad s_t=B_t^\top\Phi_{T,t+1}^\top\nabla S(x_T).
\end{equation}
This is standard differentiation through optimization \cite{hypergrad}. It is not an empirical measurement of $s_t$ in the present study. Adam-family moments are part of $x_t$ \cite{adam,adamw}; writing the full-state update as a plain parameter gradient step would omit these coordinates.

For one perturbed control, write $\widetilde S(u)=S(x_T(u))$ along the fixed branch, with other controls fixed. If $a=\langle s_t,d\rangle>0$ for a feasible direction $d$ and the Hessian of $\widetilde S$ has norm at most $M>0$ along a feasible segment $\{u_t+\epsilon d:0\le\epsilon\le\epsilon_0\}$, Taylor's theorem gives
\begin{equation}
 \widetilde S(u_t+\epsilon d)-\widetilde S(u_t)\ge\epsilon a-\tfrac12M\epsilon^2\norm d^2>0
 \quad\text{for }0<\epsilon<\min\!\left\{\epsilon_0,\frac{2a}{M\norm d^2}\right\}.
\end{equation}
A local favorable half-space must be intersected with the feasible tangent directions and a radius justified by the remainder bound. Native UID selection and acceptance logic are discrete or switching; these differentiability and relaxation assumptions do not automatically hold for the measured policy. Correct count is locally constant away from decision boundaries, so a nonzero smooth $\nabla J$ cannot be substituted for a suitable smooth utility such as $-\mathrm{CE}$.

Even in plain SGD, for a twice continuously differentiable fixed pairwise margin $m$ near $\theta$ and $\theta'=\theta-\eta g$, one has
\begin{equation}
 m(\theta')-m(\theta)=-\eta\langle\nabla m,g\rangle+O(\eta^2\norm g^2).
\end{equation}
The sign depends on whether a direction is defined as a loss gradient or an actual update. More importantly, this current sensitivity differs from the future-conditioned sensitivity $s_t$.

For softmax CE, $z=Wh$ and $p=\operatorname{softmax}(z)$ give
\begin{equation}
 \frac{\partial\ell}{\partial W_j}=(p_j-\mathbf1_{j=y})h,
 \qquad \nabla_{\theta_{\rm enc}}\ell=J_h^\top W^\top(p-e_y).
\end{equation}
Thus a replay item can affect several classes and shared features. Let $J_i=D_\theta z_i$ be a logit Jacobian (distinct from the correct-count symbol $J$). For a fixed pairwise test margin $m_{i,k}$, its one-step SGD change from replay item $j$ is
\begin{equation}
 \delta m_{i,k}=-\eta(e_{y_i}-e_k)^\top J_iJ_j^\top(p_j-e_{y_j})+O(\eta^2).
\end{equation}
The cross-Jacobian block governs the effect; different label names neither imply orthogonal updates nor determine the sign. This is a finite-network kernel expression, without an infinite-width or constant-kernel assumption \cite{ntk}. The permuted policy reinforces labels 14, 13, 15, and 21, while its five corrected reference errors occur in labels 0, 12, and 24. That excludes a strictly label-confined correctness-effect description, not all effects on reinforced-class losses or margins.

\subsection{An exact closed-loop difference identity}
\label{app:policy_difference}
Fix a random tape and a common initial state, augmented with policy memory as needed. Let $G_t^\pi$ be the complete one-step closed-loop map. For a reference policy $\beta$, define
\begin{equation}
 V_t^\beta(x)=J_Q(G_{T-1}^\beta\circ\cdots\circ G_t^\beta(x)),
 \qquad V_T^\beta=J_Q.
\end{equation}
Along the trajectory generated by $\alpha$, one has the exact identity
\begin{equation}
 \begin{split}
 J_Q(x_T^\alpha)-J_Q(x_T^\beta)
 =\sum_{t=0}^{T-1}\bigl[
 V_{t+1}^\beta(G_t^\alpha(x_t^\alpha))
 -V_{t+1}^\beta(G_t^\beta(x_t^\alpha))\bigr].
 \end{split}
 \label{eq:policydifference_exact}
\end{equation}
Indeed, the first term is $V_{t+1}^\beta(x_{t+1}^\alpha)$ and the second is $V_t^\beta(x_t^\alpha)$, so the sum telescopes. No differentiability is required. The identity requires the reference policy to be well defined on the encountered states. It formalizes the downstream aggregation in a complete-policy contrast. The attribution study and five-root policy study do not measure each counterfactual suffix on its right-hand side. Endpoint differences alone cannot assign separate causal credit to label choice, a later write, or a trigger time.

\subsection{Conditional stability, not automatic absorption}
Let $a_t,b_t\ge0$. If, for all compared steps in a region containing the trajectories,
\begin{equation}
 \norm{\delta x_{t+1}}\le a_t\norm{\delta x_t}+b_t\norm{\delta u_t},
\end{equation}
then repeated substitution yields
\begin{equation}
 \norm{\delta x_T}\le
 \left(\prod_{j=0}^{T-1}a_j\right)\norm{\delta x_0}
 +\sum_{t=0}^{T-1}b_t\norm{\delta u_t}\prod_{j=t+1}^{T-1}a_j.
\end{equation}
Empty products equal one. Small suffix products support attenuation; merely training longer does not supply that condition. Stability results for stochastic gradient methods likewise require explicit assumptions \cite{stability}. Moreover, $\norm{\Phi\delta x}$ can remain appreciable while $|\nabla S^\top\Phi\delta x|$ is small. This is readout insensitivity, not disappearance of state. The measured head separation in Appendix~\ref{app:outputgeometry} prevents using similar counts as a substitute for a contraction estimate.

\subsection{Favorable observations and expensive selection}
Smoothness alone does not turn finitely many favorable control values into a favorable interval. For example, $f(\alpha)=1-8\alpha(1-\alpha)$ is positive at both endpoints of $[0,1]$ but negative at its midpoint. More generally, smooth perturbations that vanish on all sampled controls can change unobserved regions. The saved-output segment proof works because its logit map is known to be exactly affine in $\alpha$, an assumption not established for actual training-policy interpolation.

The sufficient-quality objective in Equation~\eqref{eq:sufficient} does not require a globally maximal terminal score. Nor does having several favorable actions rule out a unique maximizer: maximizing $s^\top d$ over $\norm d\le r$ for $s\ne0$ has the unique optimum $rs/\norm s$, while many directions have positive value. For an expensive search to be preferable to a fixed structural rule, its extra value must justify its extra cost under an explicit conversion factor $\lambda$:
\begin{equation}
 V(\Pi_{\rm search})-V(\Pi_{\rm fixed})
 >\lambda\bigl[C(\Pi_{\rm search})-C(\Pi_{\rm fixed})\bigr].
\end{equation}
A naive comparison of $K$ candidate actions under $R$ future realizations and horizon $H$ requires $KRH$ ordinary rollout updates per decision, apart from cloning and evaluation. This is an accounting example, not a complexity lower bound: shared prefixes, analytic information, or other algorithms may change the cost. No global optimality gap, search-cost ratio, or value of $\lambda$ is estimated here. The observation that the archived aligned policy lies two counts below the best measured development control is not a global regret certificate or a new-root near-optimality claim.

\FloatBarrier
\section{Implementation and numerical certification}
\label{app:protocol}
\begin{table}[!htbp]\centering
\caption{Base settings for the development and five-root policy comparisons; backbone and optimizer exceptions are specified in Section~\ref{sec:protocol} and Appendix~\ref{app:extensions}. The Fiber actuator and its numerical acceptance rule are unchanged across replay and distillation.}
\label{tab:settings}
\begin{tabularx}{\linewidth}{lX}\toprule
Component&Setting\\\midrule
Model&BERT-base-cased with a growing classifier head\\
Task stream&Five tasks, ten selected classes each; five epochs per task\\
Current batch&32, except the eight-item epoch tail\\
Optimizer&AdamW \cite{adamw}; encoder LR $10^{-5}$, head LR $10^{-3}$; weight decay $5\times10^{-4}$; betas $(0.9,0.999)$, epsilon $10^{-8}$\\
Gradient clipping&Global norm 1\\
Memory/replay&150 reservoir slots; at most 16 replay examples per update\\
Evaluation&All seen heads concatenated; batch size 128; 1,000 reused-validation items at termination; the five-root policy study also uses a fixed 1,500-item test split\\
Fiber target&$\mu=1$, $\eta=1/1024$; unit head-write budget\\
Fiber schedule&Steps 191 to 799 inclusive, separated by 32 updates\\
Runtime&PyTorch 2.10.0+cu128, CUDA 12.8, Transformers 5.0.0, NumPy 2.3.5; T4 GPUs\\\bottomrule
\end{tabularx}
\end{table}

Inputs are tokenized to a maximum length of 50. Per-step batch records give 25,000 current-example presentations over 800 updates. Full replay adds 12,784 historical-example presentations, for 37,784 presentations in the merged training batches. Its first replay event is at step 1; the reservoir is empty before update 0. Distillation begins at task 1, step 160, and the first Fiber consultation is at 191.

Within each matched contrast, current inputs, memory admissions, randomization keys, and head-growth events are held fixed as specified in Section~\ref{sec:protocol}. Treatment-dependent replay selection and gradients are allowed to differ. The development replay controls reuse the archived full-replay-plus-Fiber reference. The five-root policy study executes all three complete policies from update 0 on every fresh root, with the periodic Fiber actuator present in each arm. In distillation, each branch constructs its teacher from its own preceding student; teacher generation follows the same rule but its outputs can differ after intervention.

The endpoint readouts are correct count, mean cross-entropy, and task-wise correct counts over all active classes. The 32-update four-cell experiment evaluates 400 old-task and 200 current-task items. At task ends the evaluation panel expands from 200 to 1,000 items; between-policy differences are paired within an endpoint, not between different endpoint populations. Structural endpoint diagnostics are evaluated on a copy and do not introduce an additional learning action.

\subsection{Offline numerical verification}
An internal numerical verification archive was used to audit stored readouts and algebraic certificates offline. The checks used development readouts, compact text-free logits and labels for the extensions, four-cell outcome tables, phase activation records, per-root policy outcomes, and stored frames for the mathematical calculations. The offline verifier recomputed predictions, correct counts, and mean CE where logits were stored; the initial 32-update replay records instead retained predictions and per-item losses. Contrasts were checked on common endpoint panels.

The internal verification records map each study and local root number to its original records. Workload, acceptance outcomes, calibration selection, exact source revisions, and explicit exogenous bindings were checked against those records without treating same-numbered roots in different studies as matched. Component-level restoration records and construction code support the branch-matching design; agreement of scores alone is not a matching criterion.

Offline verification compared stored numerical results and algebraic certificates against the original protocols and raw records. The internal records omit input text, tokenizer assets, pretrained weights, and complete checkpoints, and therefore do not reproduce complete GPU trajectories. Numerical verification checks consistency of the reported results; inferential scope remains determined by the experimental design.

\FloatBarrier
\section{Additional interaction results and optimizer calibration}
\label{app:extensions}
The tables report matched four-cell outcomes. Where present, $Y_{00},Y_{01},Y_{10},Y_{11}$ denote identity/off, identity/on, execute/off, and execute/on. Contrasts on $-\mathrm{CE}$ use negative cross-entropy as utility. Branch-level CE values were checked against stored branch readouts and evaluation-panel bindings during offline verification. Repeated horizons and anchors are not additional independent training roots.

\subsection{Baseline replay readouts}
\label{app:replay_readouts}
\begin{table}[H]\centering
\caption{Branch readouts underlying Figure~\ref{fig:fourcell} and Table~\ref{tab:taskeffects}, at a 32-update horizon. All branches share the same pre-action parent, proposal specification, and 600-item panel. The last column records a diagnostic on a copy of the endpoint.}
\label{tab:h32}
\begin{tabular}{llrrrr}\toprule
Anchor action&Future replay&$J$&Old-task $J$&$m_{11}$&Would commit\\\midrule
Execute&On&578&385&7.6531&No\\
Identity&On&578&385&7.1279&No\\
Execute&Off&500&305&0.4308&Yes\\
Identity&Off&495&300&$-0.7705$&Yes\\\bottomrule
\end{tabular}
\end{table}

\subsection{Distillation and backbone extensions}
The distillation parents come from the first three legal commits of the LwF-plus-Fiber development prefix. Prefix writes prior to each anchor are retained as part of that parent's history; branch continuations have no later writes. The optimizer comparisons instead use unactuated scouts. The RoBERTa study uses the first legal nonzero event on the period-32 schedule by 671 in each new root, with no outcome-based replacement.

\begin{table}[!htbp]\centering\small
\caption{Full distillation four-cell readouts. One development root, three anchors, two horizons; panels expand where shown.}
\label{tab:distillation_full}
\begin{tabular}{lrrrrrrrrr}\toprule
Root&Anchor&$H$&$n$&$Y_{00}$&$Y_{01}$&$Y_{10}$&$Y_{11}$&$I_J$&$I_{-\mathrm{CE}}$\\\midrule
Dev. & 191 & 32 & 400 & 315 & 346 & 315 & 349 & $+3$ & $+0.015386$\\
Dev. & 191 & 128 & 400 & 305 & 305 & 305 & 306 & $+1$ & $+0.006555$\\
Dev. & 223 & 32 & 400 & 318 & 335 & 322 & 341 & $+2$ & $-0.041508$\\
Dev. & 223 & 128 & 600 & 372 & 458 & 376 & 458 & $-4$ & $-0.014583$\\
Dev. & 255 & 32 & 400 & 299 & 320 & 299 & 325 & $+5$ & $+0.000954$\\
Dev. & 255 & 128 & 600 & 336 & 375 & 344 & 376 & $-7$ & $+0.003637$\\
\bottomrule\end{tabular}
\end{table}

\begin{table}[!htbp]\centering\small
\caption{Full RoBERTa readouts. The three root numbers are local to the RoBERTa comparison. A positive interaction need not mean a positive Fiber effect with replay.}
\label{tab:backbone_full}
\begin{tabular}{lrrrrrrrrr}\toprule
Root&Anchor&$H$&$n$&$Y_{00}$&$Y_{01}$&$Y_{10}$&$Y_{11}$&$I_J$&$I_{-\mathrm{CE}}$\\\midrule
1 & 511 & 32 & 800 & 689 & 752 & 688 & 753 & $+2$ & $+0.013963$\\
1 & 511 & 128 & 800 & 660 & 752 & 655 & 752 & $+5$ & $+0.026789$\\
2 & 511 & 32 & 800 & 689 & 761 & 688 & 761 & $+1$ & $+0.043752$\\
2 & 511 & 128 & 800 & 667 & 759 & 665 & 760 & $+3$ & $+0.046859$\\
3 & 575 & 32 & 800 & 752 & 769 & 753 & 767 & $-3$ & $+0.033377$\\
3 & 575 & 128 & 1,000 & 820 & 954 & 816 & 952 & $+2$ & $+0.027539$\\
\bottomrule\end{tabular}
\end{table}

\subsection{Optimizer-native primary comparison}
The optimizer-native study calibrates SGDW on a separate calibration root using the ordered encoder/head pairs $(10^{-4},10^{-2})$, $(3\cdot10^{-4},3\cdot10^{-2})$, $(10^{-3},10^{-1})$, and $(3\cdot10^{-3},3\cdot10^{-1})$. It maximizes Task-0 correct count, then minimizes CE, then takes the lower candidate index in a tie. The fourth pair is selected, with $J=199/200$ and CE $0.0082043$. Its three measurement roots are numbered 1--3 within this study. AdamW uses encoder/head rates $(10^{-5},10^{-3})$, betas $(0.9,0.999)$, and epsilon $10^{-8}$. Both conditions use weight decay $5\cdot10^{-4}$; SGDW uses momentum $0.9$ and no Nesterov acceleration.

All six Task-2 scans return saturated identity at all five prescribed steps. All six Task-3 scans select step 511, leaving the later phase opportunities unsearched. The two optimizer conditions share each root's construction and specified stochastic inputs, not the full optimizer-specific parent states. The six phase nonactivations have no execute/off--on outcomes and are not assigned fabricated zero contrasts.

\begin{table}[!htbp]\centering\footnotesize
\caption{Full optimizer-native readouts for the six activated Task-3 parents. Both horizons use the same 800-item panel. Root numbers refer only to this study; optimizer conditions are paired within a root.}
\label{tab:native_optimizer_full}
\begin{tabular}{llrrrrrrrrr}\toprule
Root&Optimizer&Anchor&$H$&$n$&$Y_{00}$&$Y_{01}$&$Y_{10}$&$Y_{11}$&$I_J$&$I_{-\mathrm{CE}}$\\\midrule
1 & AdamW & 511 & 32 & 800 & 693 & 747 & 697 & 749 & $-2$ & $-0.024153$\\
1 & AdamW & 511 & 128 & 800 & 663 & 753 & 667 & 751 & $-6$ & $-0.048234$\\
1 & SGDW & 511 & 32 & 800 & 704 & 735 & 710 & 741 & $0$ & $-0.105356$\\
1 & SGDW & 511 & 128 & 800 & 673 & 739 & 675 & 740 & $-1$ & $-0.118092$\\
2 & AdamW & 511 & 32 & 800 & 694 & 755 & 702 & 757 & $-6$ & $-0.017977$\\
2 & AdamW & 511 & 128 & 800 & 651 & 754 & 659 & 754 & $-8$ & $-0.043018$\\
2 & SGDW & 511 & 32 & 800 & 678 & 733 & 679 & 732 & $-2$ & $-0.024880$\\
2 & SGDW & 511 & 128 & 800 & 642 & 741 & 643 & 740 & $-2$ & $-0.029122$\\
3 & AdamW & 511 & 32 & 800 & 717 & 741 & 717 & 741 & $0$ & $-0.009414$\\
3 & AdamW & 511 & 128 & 800 & 691 & 744 & 691 & 744 & $0$ & $-0.009821$\\
3 & SGDW & 511 & 32 & 800 & 688 & 722 & 691 & 725 & $0$ & $-0.054441$\\
3 & SGDW & 511 & 128 & 800 & 653 & 731 & 655 & 730 & $-3$ & $-0.067756$\\
\bottomrule\end{tabular}
\end{table}

\subsection{Independent-rate update-scale sensitivity}
\label{app:optimizer_calibration}
The update-scale sensitivity study uses three measurement roots distinct from those of the optimizer-native comparison. The two diagnostic designs in Appendix~\ref{app:optimizer_diagnostics} reuse this same three-root panel. Their outcomes therefore do not constitute independent-root replications, and equal root numbers across the two principal optimizer studies carry no pairing information.

The update-scale sensitivity study independently varies encoder rates in $\{0.005,0.01,0.02,0.03\}$ and head rates in $\{0.003,0.006,0.01,0.02\}$, a frozen $4\times4$ grid on a separate calibration root. The AdamW reference attains $197/200$ Task-0 correct predictions. A candidate must finish within two counts of this reference. At steps 31, 63, 95, 127, and 159, separate head and sampled-encoder update-RMS ratios to the reference define
\begin{equation}
 e_c=\frac15\sum_{s}\left|\log\frac{u_{c,s}^{\rm SGDW}}{u_{c,s}^{\rm AdamW}}\right|,
 \qquad c\in\{\mathrm{enc},\mathrm{head}\}.
\end{equation}
The encoder probe uses three fixed parameter matrices in layers 0, 5, and 11, not the complete vector field. The rule minimizes $\max(e_{\rm enc},e_{\rm head})$ among quality-eligible candidates; ties use mean channel score, absolute Task-0 CE difference, and candidate index. Both scores must be at most $\log3$ before interaction measurement proceeds.

The selected pair $(0.01,0.02)$ achieves $198/200$ and CE $0.0139232$, with factors $\exp(e_{\rm enc})=1.435821$ and $\exp(e_{\rm head})=1.198306$. The gate is an aggregate sampled-channel scale criterion, not per-coordinate equality or a match of optimizer moments. Four fixed-anchor cells activate; Root~3 at 351 and 511 returns saturated identity. This is complementary sensitivity evidence under a specified scale criterion, rather than a match of full trajectories or a comparison of same-numbered roots with the native-optimizer study.

\begin{table}[!htbp]\centering\small
\caption{Complete update-scale sensitivity outcomes at the four activated parents. Root~3 has two nonactivations and no corresponding outcome rows. The sign change for Root~1 at anchor 511 uses a common panel across horizons.}
\label{tab:scale_sensitivity_full}
\begin{tabular}{lrrrrrrrrr}\toprule
Root&Anchor&$H$&$n$&$Y_{00}$&$Y_{01}$&$Y_{10}$&$Y_{11}$&$I_J$&$I_{-\mathrm{CE}}$\\\midrule
1 & 351 & 32 & 600 & 436 & 574 & 435 & 573 & $0$ & $+0.008095$\\
1 & 351 & 128 & 600 & 413 & 564 & 414 & 564 & $-1$ & $+0.003807$\\
1 & 511 & 32 & 800 & 599 & 755 & 602 & 754 & $-4$ & $+0.023977$\\
1 & 511 & 128 & 800 & 564 & 759 & 563 & 761 & $+3$ & $+0.021861$\\
2 & 351 & 32 & 600 & 437 & 571 & 437 & 571 & $0$ & $+0.000479$\\
2 & 351 & 128 & 600 & 412 & 578 & 412 & 579 & $+1$ & $+0.000080$\\
2 & 511 & 32 & 800 & 637 & 749 & 637 & 749 & $0$ & $-0.000084$\\
2 & 511 & 128 & 800 & 585 & 758 & 585 & 758 & $0$ & $-0.000356$\\
\bottomrule\end{tabular}
\end{table}

\FloatBarrier
\section{Coordination-policy results}
\label{app:support}
These complete-policy comparisons allow later states, accepted writes, and allocations to evolve within each arm. They are distinct from the common-parent four-cell experiments in Appendix~\ref{app:extensions}.

\subsection{Replay-retention comparisons}
\begin{table}[H]\centering
\caption{Terminal outcomes under different retained replay fractions. Each row has its own training history, so comparisons across rows do not isolate the effect of continuation at a fixed state. These are descriptive results, not a jointly randomized dose-response estimate.}
\label{tab:retention}
\begin{tabular}{rrrr}\toprule
Replay retained (\%)&Without Fiber $J$&With Fiber $J$&Difference\\\midrule
100&950&950&0\\
75&948&941&$-7$\\
50&942&938&$-4$\\
25&934&933&$-1$\\
12.5&929&926&$-3$\\
0&424&507&$+83$\\\bottomrule
\end{tabular}
\end{table}

On the 75\%-retained trajectory, let $\Pi^{(i)}$ retain its first $i$ accepted writes and disable every later write. Each pair branches from the corresponding parent reconstructed on the original trajectory, so $E_{351},E_{511},E_{671},E_{703}$ evaluate $\Pi^{(1)},\ldots,\Pi^{(4)}$, and their identity branches evaluate $\Pi^{(0)},\ldots,\Pi^{(3)}$. The measured differences are $0,-3,-1,-3$, totaling $-7$. This policy construction licenses Equation~\eqref{eq:telescoping}; equality of two summary scores alone would not. As a numerical cross-check, each preceding execute endpoint and the next identity endpoint have identical saved logits, predictions, correctness, and CE arrays on the common terminal panel. These equalities do not by themselves assert equality of unobserved checkpoint coordinates. The values remain specific to this nesting and continuation.

\subsection{Task-end performance of the coordination policies}
\begin{table}[H]\centering
\caption{Correct counts at task ends. The evaluation sizes are respectively 200, 400, 600, 800, and 1,000. A terminal improvement need not imply non-degradation at all earlier endpoints.}
\label{tab:taskends}
\begin{tabular}{lrrrrr}\toprule
Policy&159&319&479&639&799\\\midrule
Full replay + Fiber&196&393&573&751&950\\
Replay withdrawal&196&393&495&754&851\\
Replay deprioritization&196&393&569&748&948\\
Fiber-aligned reinforcement&196&393&574&750&953\\\midrule
Mask-matched random&196&393&573&751&953\\
Permuted signal&196&393&574&752&955\\
Shadow signal / no write&196&393&574&748&952\\\bottomrule
\end{tabular}
\end{table}

The full reinforcement policy has paired differences $(0,0,+1,-1,+3)$ relative to full replay with Fiber. The final gain is therefore a terminal observation, not a pathwise safety guarantee. The single-window comparison at $H=32$ gives $577$ versus $578$ on 600 items, while increasing the label-11 margin from $7.6531$ to approximately $7.7876$. Because the full policy permits further Fiber actions and allocation windows, the two contrasts concern different interventions rather than only different observation times.

\subsection{External workload and retained writes}
\begin{table}[H]\centering
\caption{External workload and retained state writes for Figure~\ref{fig:policies}. All rows except withdrawal use 799 replay events and 12,784 replay examples. The five reallocation policies make 96 or 112 replacements as described in Section~\ref{sec:coord}. The shadow policy accepts three proposals but retains none; the last column counts retained writes, not accepted proposals or total sensing cost.}
\label{tab:policies}
\begin{tabular}{lrrrr}\toprule
Policy&Terminal $J$&Replay events&Replay examples&Retained writes\\\midrule
Full replay + Fiber&950&799&12,784&3\\
Replay withdrawal&851&511&8,176&10\\
Replay deprioritization&948&799&12,784&10\\
Fiber-aligned reinforcement&953&799&12,784&3\\
Mask-matched random&953&799&12,784&3\\
Permuted signal&955&799&12,784&3\\
Shadow signal / no write&952&799&12,784&0\\\bottomrule
\end{tabular}
\end{table}

\subsection{Five-root policy comparison}
\label{app:policy_details}
Roots~1--5 are local to the five-root policy study. $S$, $G$, and $M$ denote standard full replay plus period-32 Fiber, guided reinforcement plus period-32 Fiber, and mask-matched random content plus period-32 Fiber. They share each root's initial conditions and exogenous bindings, not the subsequent endogenous state. All 15 full runs complete 800 ordinary updates with 799 replay events and 12,784 replay examples. The fresh-test split is used only at termination.

\begin{table}[!htbp]\centering\small
\caption{All terminal fresh-test counts and mean CE. The 1,500-item panel is common across the 15 trajectories. CE is secondary; lower is better.}
\label{tab:policy_test_full}
\begin{tabular}{lrrrrrr}\toprule
Root&$J_S$&$J_G$&$J_M$&$\mathrm{CE}_S$&$\mathrm{CE}_G$&$\mathrm{CE}_M$\\\midrule
1 & 1343 & 1344 & 1341 & 0.4558487 & 0.4449773 & 0.4508730\\
2 & 1351 & 1348 & 1350 & 0.4481850 & 0.4477481 & 0.4501625\\
3 & 1323 & 1327 & 1324 & 0.4612834 & 0.4606896 & 0.4618068\\
4 & 1363 & 1360 & 1363 & 0.3936347 & 0.3930567 & 0.3919950\\
5 & 1346 & 1344 & 1347 & 0.4147207 & 0.4141283 & 0.4158862\\
\bottomrule\end{tabular}
\end{table}

\begin{table}[!htbp]\centering\small
\caption{Terminal reused-validation readouts on 1,000 items. These do not override the fresh-test primary verdict.}
\label{tab:policy_validation_full}
\begin{tabular}{lrrrrrr}\toprule
Root&$J_S$&$J_G$&$J_M$&$\mathrm{CE}_S$&$\mathrm{CE}_G$&$\mathrm{CE}_M$\\\midrule
1 & 915 & 915 & 917 & 0.3796341 & 0.3671043 & 0.3744089\\
2 & 919 & 918 & 922 & 0.3118337 & 0.3138216 & 0.3133275\\
3 & 914 & 916 & 913 & 0.3712733 & 0.3682319 & 0.3737126\\
4 & 928 & 927 & 930 & 0.2735706 & 0.2711383 & 0.2722694\\
5 & 918 & 917 & 921 & 0.3362293 & 0.3340249 & 0.3362681\\
\bottomrule\end{tabular}
\end{table}

On reused validation, $G-S=(0,-1,+2,-1,-1)$, with mean $-0.2$ and median $-1$; it is mixed/inconclusive. $G-M=(-2,-4,+3,-3,-4)$ has mean $-2$ and median $-3$, satisfying the descriptive directionally-negative rule. $M-S=(+2,+3,-1,+2,+3)$ has mean $1.8$ and median $2$, satisfying the directionally-positive rule. On fresh test, however, $M-S=(-2,-1,+1,0,+1)$ is mixed, with mean $-0.2$ and median zero. No uniformly better content policy follows from these panels.

\begin{table}[!htbp]\centering\small
\caption{\textbf{Endogenous actions and actual reallocation workload.} Each arm has 20 consultations. ``Duplicate unchanged'' counts active steps on which a duplicate-UID native candidate is left intact under the specified duplicate-candidate rule; the ordinary update is still executed. Redirect counts need not match across divergent whole trajectories.}
\label{tab:policy_actions}
\begin{tabular}{llL{0.27\linewidth}rrr}\toprule
Root&Arm&Commit steps&Writes&Redirected rows&Duplicate unchanged\\\midrule
1 & S & 511, 543, 671 & 3 & 0 & 0\\
1 & G & 511, 671 & 2 & 59 & 2\\
1 & M & 511, 543, 671 & 3 & 86 & 5\\
2 & S & 511 & 1 & 0 & 0\\
2 & G & 511 & 1 & 32 & 0\\
2 & M & 511 & 1 & 32 & 0\\
3 & S & 511 & 1 & 0 & 0\\
3 & G & 511, 671 & 2 & 63 & 1\\
3 & M & 511 & 1 & 32 & 0\\
4 & S & 351, 511, 671 & 3 & 0 & 0\\
4 & G & 351, 511, 671 & 3 & 90 & 5\\
4 & M & 351, 511, 671 & 3 & 90 & 5\\
5 & S & 351, 671 & 2 & 0 & 0\\
5 & G & 351, 671 & 2 & 60 & 4\\
5 & M & 351, 671 & 2 & 60 & 4\\
\bottomrule\end{tabular}
\end{table}

The random arm's change mask is matched to its own shadow evaluation of the guided primitive; it is not imposed from the guided arm's realized trajectory. Exact own-trajectory mask agreement is recorded in all five random runs. Different future commits and eligible donor availability make whole-arm redirected totals differ. No globally unique UID invariant is imposed on native reservoir replay; duplicate candidates are preserved rather than deduplicated. The duplicate-candidate rule was fixed before the reported complete trajectories. All five specified roots and all three policy arms are reported; implementation provenance was retained and checked in internal verification records.

\subsection{Complete-policy distillation comparison}
\label{sec:lwf}
This complete-policy comparison tests LwF-style output distillation with and without the unchanged Fiber actuator over 800 updates. At each task boundary, before head growth, the branch's student becomes a frozen teacher for the next task. Teacher targets are evaluated on current-task inputs \cite{lwf}. The loss is Equation~\eqref{eq:lwf}, with $T=2$ and $\lambda=1$. This is the implemented LwF-style rule, not an optimized comparison across distillation variants. Both arms make 640 teacher forwards and use zero replay examples in the ordinary training loss. The Fiber actuator nevertheless accesses historical reservoir representatives, so the treatment is not memory-free or matched to the baseline in actual historical-information use. Branch-local teachers can diverge after Fiber writes; their refresh rule, not their outputs, is held fixed.

Table~\ref{tab:lwf} records terminal improvement from $414$ to $433$ and a CE reduction of $0.2663619$. The task-wise changes $(+22,+25,-7,-21,0)$ expose the nonuniform trade-off: the net $+19$ does not improve every task. The paired task-end differences are $(0,+21,+72,+115,+19)$ on their respective growing panels; terminal item flips comprise 57 gains and 38 losses.

\begin{table}[!htbp]\centering
\caption{\textbf{A complete-policy distillation outcome with task-level trade-offs.} Each task has 200 terminal evaluation items. Teacher-forward counts are equal, but complete cost and historical-information use are not. This is a two-policy comparison, not a four-cell distillation-interaction identification.}
\label{tab:lwf}
\begin{tabular}{lrrr}\toprule
Readout & LwF-only & LwF + Fiber & Difference\\\midrule
Task 0 correct&40&62&$+22$\\
Task 1 correct&16&41&$+25$\\
Task 2 correct&65&58&$-7$\\
Task 3 correct&96&75&$-21$\\
Task 4 correct&197&197&0\\\midrule
Total correct&414&433&$+19$\\
Mean CE&2.2769445&2.0105826&$-0.2663619$\\
Teacher forwards&640&640&0\\\bottomrule
\end{tabular}
\end{table}

\subsection{Continuation-dependent actuation profiles}
Table~\ref{tab:demand} summarizes all eligible Fiber decisions. The full-replay and Fiber-aligned reinforcement trajectories each make three commits; the LwF trajectory commits at all 15 points from 191 through 639, with summed committed write norms approximately $6.4961$. The final five consultations return rejected identity under the ideal-write cap. As Section~\ref{sec:actuator} establishes, this means the specified numerical route did not certify the prescribed map within budget. It is not a pure deficiency-severity measure or proof that every useful action was infeasible.

\begin{table}[!htbp]\centering
\caption{\textbf{Observed decisions under the same 20-point consultation schedule.} Consultations occur at steps 191, 223, $\ldots$, 799. The replay trajectories share decision times, not identical hidden states. Every LwF rejection is due to the ideal-write cap.}
\label{tab:demand}
\begin{tabular}{lrrr}\toprule
External continuation & Accepted repair & Saturated identity & Rejected identity\\\midrule
Full replay&3&17&0\\
Fiber-aligned reinforcement&3&17&0\\
LwF-style distillation&15&0&5\\\bottomrule
\end{tabular}
\end{table}

The late rejections coincide with a decline in the paired advantage, but neither a cap intervention nor a fixed-panel temporal analysis identifies that decline's cause. The positive terminal contrast supports a complete-policy application under distillation, with absolute accuracy still substantially below replay. These two LwF arms measure policy performance and actuation demand. Continuation dependence of a fixed proposal is identified by the separate four-cell experiment in Section~\ref{sec:distillationinteraction}.

\FloatBarrier
\section{Geometric analysis and optimizer diagnostics}
\label{app:outputgeometry}
The geometric analyses use stored frames and outputs from the allocation-attribution controls, without further model training or selection of an interpolation parameter for deployment. Stored raw logits were used for the mask-matched random, permuted-signal, and shadow/no-write controls. The archived aligned endpoint is a separate 953 reference. The final subsection reports auxiliary optimizer diagnostics on the update-scale sensitivity panel.

\subsection{Exact rank of the stored frame and the ideal repair set}
Treat each stored FP32 entry as its exact dyadic rational. Its denominator is a power of two, invertible modulo the odd prime 65521. Gaussian elimination over that finite field yields full row rank for the current and stacked frames of all nine accepted proposals in the three controls. A nonzero maximal minor modulo 65521 is nonzero over the rationals, proving the corresponding real row rank of these stored matrices; this is not a numerical singular-value threshold argument.

At the common step-351 frame, $C$ has 8 rows and $[C;\mathsf H]$ has rank 28, with $q=30$ and $d=769$. The affine dimensions in Corollary~\ref{cor:repairset} are therefore $22,230$ and, with centering, $21,489$. The independently recomputed ideal minimum norm is $0.068599529159949$; the archived numerical upper bound is $0.06859952917287827$. Substituting that bound into the unit-cap radius yields approximately $0.99764428$ as a lower-radius estimate for the \emph{ideal stored-frame set}. It is not a certified deployment radius for arbitrary rounded writes. The original runtime follows its one prescribed proposal path and native checks.

\subsection{Output superlevel sets and common-correct certificates}
For a fixed item set $S$, every inequality in Equation~\eqref{eq:outputcell} is strict and affine in $L$, so $\mathcal C_S$ is open and convex. If $L_0,L_1\in\mathcal C_S$, then for $0\le\alpha\le1$,
\begin{equation}
 \bigl[(1-\alpha)L_0+\alpha L_1\bigr]_{i,y_i}
 -\bigl[(1-\alpha)L_0+\alpha L_1\bigr]_{i,k}
 =(1-\alpha)m_{i,k}(L_0)+\alpha m_{i,k}(L_1)>0.
\end{equation}
Thus every common strictly correct item stays correct on the entire segment. In general $J\ge k$ can be represented, away from top-logit ties, as a union over size-$k$ sets $S$ of these regions; that union need not be convex.

The random and permuted endpoints share 951 correct items, giving the immediate lower bound $J\ge951$ on their segment. To sharpen it, we solve every correct-class versus competitor margin equation along each segment. Its coefficients are dyadic rationals, so all breakpoints are computed exactly as rational numbers. We count predictions on every open subinterval and at each breakpoint under the stored lowest-index argmax tie convention. This is exhaustive one-dimensional boundary analysis, not sampling a grid.

\begin{table}[!htbp]\centering
\caption{\textbf{Exact saved-logit segment certificates on 1,000 items.} The endpoints are three measured attribution controls. The counts apply to output mixtures, not parameter or replay-policy mixtures. Breakpoints include the segment endpoints.}
\label{tab:segments}
\begin{tabular}{lrrr}\toprule
Output segment & Minimum $J$ & Maximum $J$ & Breakpoints\\\midrule
Random--permuted &953&955&16\\
Random--shadow &952&954&31\\
Permuted--shadow &952&955&23\\\bottomrule
\end{tabular}
\end{table}

The union of the three segments is a path-connected subset with $J\ge952>950$. This does not certify the whole triangle's count minimum. For CE, by contrast, convexity of log-sum-exp gives
\begin{equation}
 \operatorname{CE}\!\left(\sum_j\alpha_jL_j\right)
 \le\sum_j\alpha_j\operatorname{CE}(L_j),
 \qquad \alpha_j\ge0,\quad\sum_j\alpha_j=1.
 \label{eq:ceconvex}
\end{equation}
All three control CEs are below the reference $0.2195535$: random $0.2166514$, permuted $0.2108371$, and shadow $0.2102889$. Their entire logit convex hull therefore has lower CE than that reference. Shadow has the lowest mean CE among these three but the lowest correct count, illustrating the distinction between the two readouts without inferring a common causal mechanism.

For a single endpoint, let $\gamma$ be the smallest correct-versus-competitor margin among its strictly correct items. Any perturbation $E$ with $\norm E_{\max}<\gamma/2$ preserves those items, because a pairwise margin changes by at most $2\norm E_{\max}$. The corresponding radii are $0.0111651421$, $0.0124964714$, and $0.0344238281$ for random, permuted, and shadow. These are local favorable \emph{output} neighborhoods, not special properties of Fiber. Continuity of a model's output map can give preimages of such neighborhoods; a quantitative parameter radius or training reachability requires additional information not established here.

\subsection{Head separation under a common ordinary suffix}
After the last reallocation window ends at step 703, the stored exogenous records agree on current UIDs, ordinary replay UIDs, admission state hashes, and random keys. Offers at 735, 767, and 799 are identities. The measured pre-offer head distances are:
\begin{center}
\begin{tabular}{lrrrr}\toprule
Step &703&735&767&799\\\midrule
$\norm{\Theta_M-\Theta_P}_F$&0.24563454&0.25139960&0.25226788&0.25537934\\\bottomrule
\end{tabular}
\end{center}
These values neither show disappearance nor monotone contraction of the head difference. They do not exclude decay in another state component or a readout-sensitive projection. Comparing only $J_M=953$ and $J_P=955$ cannot distinguish state contraction from output insensitivity.

\subsection{What the reallocation fraction does not bound}
Each control replaces 96 of 12,784 replay presentations, $0.75094\%$, relative to its own ordinary candidates. The random--permuted realized slot Hamming distance is instead 179: the policies can choose different donor positions. Random--shadow differs at 182 slots and permuted--shadow at 155. The candidate sequences are common; the modified sequences need not be. A slot fraction is not a perturbation norm or a utility bound.

At a common state with fixed example randomness and before clipping, replacing one row in a mean loss changes the gradient by
\begin{equation}
 \delta g_t=\frac{g^{\rm new}_t-g^{\rm old}_t}{b_t},
 \qquad b_t=n_{\rm current}+n_{\rm replay}.
\end{equation}
Here $b_t$ is usually 48 and is 24 at the short epoch tail, not 16. Once states diverge, gradients of all the unchanged rows may differ as well. Clipping and AdamW moments further prevent an interpretation based on the fraction alone.

\subsection{Native-event and linked-rate fixed-anchor diagnostics}
\label{app:optimizer_diagnostics}
The native-event diagnostic uses momentum SGD with coupled weight decay and selects rates $(0.001,0.1)$ by Task-0 quality from the same four rate pairs used by the native calibration. It observes the first legal event at 191 in all three measurement roots. The fixed-anchor diagnostic uses decoupled SGDW, unactuated prefixes, and anchors 351 and 511. Its calibration selects the linked-rate pair $(0.003,0.3)$ using sampled update scales under a Task-0 quality guard, but does not adequately match both channels. Only Root~1 activates, at both anchors; Roots~2 and~3 nonactivate at both.

These diagnostics retain information about recipe-dependent scale and activation. Their differing optimizer, calibration, and anchor conditions do not isolate the cause of an activation difference. In particular, sharing a root panel does not turn them into repeated measurements of one unchanged intervention design.

\begin{table}[!htbp]\centering\small
\caption{Optimizer diagnostic readouts. The native-event, coupled-decay design has three activated roots. The fixed-anchor, linked-rate design has two activated cells and four nonactivations. Both reuse the sensitivity-study roots; they are not additional independent roots for the native-optimizer comparison.}
\label{tab:optimizer_diagnostics}
\begin{tabular}{lrrrrrrrrr}\toprule
Root&Anchor&$H$&$n$&$Y_{00}$&$Y_{01}$&$Y_{10}$&$Y_{11}$&$I_J$&$I_{-\mathrm{CE}}$\\\midrule
\multicolumn{10}{l}{Native-event diagnostic: coupled-decay SGD}\\
1 & 191 & 32 & 400 & 380 & 387 & 380 & 387 & $0$ & $-0.000007$\\
1 & 191 & 128 & 400 & 366 & 393 & 366 & 393 & $0$ & $-0.000023$\\
2 & 191 & 32 & 400 & 381 & 389 & 380 & 389 & $+1$ & $+0.000593$\\
2 & 191 & 128 & 400 & 357 & 392 & 357 & 392 & $0$ & $-0.001189$\\
3 & 191 & 32 & 400 & 382 & 390 & 382 & 390 & $0$ & $-0.000472$\\
3 & 191 & 128 & 400 & 362 & 394 & 362 & 394 & $0$ & $-0.000848$\\
\midrule\multicolumn{10}{l}{Fixed-anchor diagnostic: linked-rate SGDW}\\
1 & 351 & 32 & 600 & 536 & 571 & 536 & 571 & $0$ & $-0.003822$\\
1 & 351 & 128 & 600 & 509 & 563 & 509 & 563 & $0$ & $-0.000195$\\
1 & 511 & 32 & 800 & 686 & 739 & 686 & 739 & $0$ & $-0.000959$\\
1 & 511 & 128 & 800 & 642 & 741 & 641 & 739 & $-1$ & $-0.001171$\\
\bottomrule\end{tabular}
\end{table}

\end{document}